\documentclass{article}

\usepackage{arxiv}

\usepackage[utf8]{inputenc} 
\usepackage[T1]{fontenc}    
\usepackage{hyperref}       
\usepackage{url}            
\usepackage{booktabs}       
\usepackage{amsfonts}       
\usepackage{nicefrac}       
\usepackage{microtype}      
\usepackage{lipsum}		
\usepackage{graphicx}
\usepackage{natbib}
\usepackage{doi}

\usepackage{subcaption}
\usepackage{dblfloatfix} 
\usepackage{enumitem}
\usepackage{todonotes}
\usepackage{times}
\usepackage{soul}
\usepackage{url}
\usepackage{booktabs}
\usepackage{multirow}
\usepackage{makecell}
\usepackage{hyperref}
\usepackage[utf8]{inputenc}
\usepackage{graphicx}
\usepackage{amsmath,amssymb,amsthm}
\usepackage{booktabs}
\usepackage{algorithm}
\usepackage[switch]{lineno}
\usepackage{pifont}
\usepackage{algorithm}
\usepackage{algpseudocode}
\usepackage{fancyhdr}
\usetikzlibrary{decorations.pathmorphing}
\usepackage{amsmath}
\usepackage{tikz}
\usetikzlibrary{decorations.pathreplacing,arrows.meta,calc}
\usetikzlibrary{arrows.meta,calc,positioning}

\DeclareMathOperator{\KL}{KL}
\DeclareMathOperator*{\argmin}{arg\,min}

\newtheorem{theorem}{Theorem}[section]
\newtheorem{proposition}[theorem]{Proposition}

\newtheorem{lemma}[theorem]{Lemma}
\theoremstyle{definition}
\newtheorem{definition}[theorem]{Definition}
\theoremstyle{remark}
\newtheorem{remark}[theorem]{Remark}
\newcommand{\indep}{\perp\!\!\!\perp}

\newcommand{\Mks}{M_{K^*}}

\newcommand{\eps}{\varepsilon}

\makeatletter
\newenvironment{breakablealgorithm}
  {%
   \begin{center}
     \refstepcounter{algorithm}%
     \hrule height.8pt depth0pt \kern2pt%
     \renewcommand{\caption}[2][\relax]{%
       {\raggedright\textbf{Algorithm \thealgorithm} ##2\par}%
       \ifx\relax##1\relax 
       \addcontentsline{loa}{algorithm}{\protect\numberline{\thealgorithm}##2}%
       \else 
       \addcontentsline{loa}{algorithm}{\protect\numberline{\thealgorithm}##1}%
       \fi
       \kern2pt\hrule\kern2pt
     }%
  }{%
     \kern2pt\hrule\relax%
   \end{center}
  }
\makeatother

\title{Global Explanations for Multivariate Time Series Forecasting Models via $K$-Order Markov Approximations}

\author{ \href{https://orcid.org/0009-0008-5195-090X}{\includegraphics[scale=0.06]{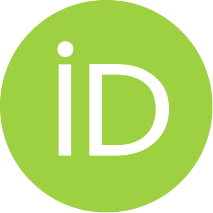}\hspace{1mm}Amadeo Tunyi} \\
	XITASO GmbH \\
	Austraße 35, 86153 Augsburg, Germany \\
	\texttt{amadeo.tunyi@xitaso.com} \\
}

\begin{document}
\maketitle

\begin{abstract}
	While many explainable AI (XAI) methods have been
proposed, most are not designed for time-series forecasting models and
often rely on the implicit assumption that timestamp features are
independent. This assumption ignores the fundamental property of temporal
dependence and can lead to explanations that violate the sequential and
causal structure of the data. We introduce \textsc{KARMA}, a method for explaining time-series predictors
by constructing a Markov surrogate model that captures the temporal
dependencies learned by the predictor. Our approach revolves around three main aspects: identifying the minimal history length $K$ that is
predictively sufficient for the model, estimating the best-fitting
$K$-order Markov transition kernel from the discretized history space, and a five-level global explanation hierarchy that can be derived from the Markov transition kernel, which we illustrate using real-world weather data (Beijing PM 2.5). We also certify using complex synthetic data with known true causal edges that KARMA (i) recovers the data causal structure as learned by the model via a controlled experiment and (ii) identifies temporal dependencies better than established attribution methods such as TimeSHAP.
\end{abstract}

\section{Introduction}

The deployment of deep learning models on financial time series, clinical
monitoring and industrial sensing have grown substantially. Recurrent networks,
temporal convolutional networks (TCNs), and transformer architectures achieve
strong predictive performance, but remain largely opaque: practitioners cannot
readily identify which temporal patterns, cross-variable dependencies, or
historical regimes drive a particular prediction. This opacity creates
regulatory friction under ESMA guidelines and the EU AI Act, and undermines
model trust in high-stakes decision environments.

Explainable AI (XAI) methods such as LIME \cite{ribeiro2016}, SHAP
\cite{lundberg2017}, and attention-based attribution were developed primarily
for static, i.i.d.\ settings. Their extension to time series is non-trivial:
temporal autocorrelation, non-stationarity, and inter-variable Granger
causality violate the independence assumptions underlying most attribution
frameworks. For instance, SHAP's baseline marginalisation is structurally incoherent for
time series because $X_{t-2}$ is not independent of $X_{t-1}$; gradient
attributions describe local geometry at a single point rather than the
systematic conditional structure of what the model has globally learned.

We propose a different approach rooted in probabilistic approximation. Rather
than perturbing inputs or computing gradient-based attributions, we ask:
\emph{can the behaviour of a black-box model on a multivariate time series be
faithfully approximated by a $K$-th order Markov chain?} If so, the transition
probabilities of that chain constitute a structured, interpretable explanation
grounded in the conditional dependence structure of the data as learned by the model. This framing
is natural for sequential domains: a clinician, engineer, or trader already
reasons in conditional scenarios ``given the last three states of the
system, what does the model expect next?'' and transition probabilities
answer that question directly, in the native language of the domain. We present KARMA, an approach to interpreting a black-box time-series model using a five-level (global) explanation hierarchy derived from the transition kernel. Our contribution is organized around three distinct pillars with different
theoretical foundations and different data requirements.

\begin{enumerate}
    \item \textit{Pillar~1 (Markov surrogate).} A surrogacy predictive validity stopping rule selects the minimal lag $K^*$ that is predictively sufficient, and for this $K^*$ estimate with minimal error using the data, the best nearest $K^*-$Order Markov probability transition kernel. This answers the question every practitioner should ask before trusting a sequential model: how much of its input window does the model actually use? The answer is both certified and model-agnostic,
    
    \item \textit{Pillar~2 (Compression and certified attribution).} When $K$ is smaller than the model window size, the model is provably insensitive to inputs beyond lag $K^*$. This yields a compression ratio, a model-certified baseline $b^*$ that resolves the baseline selection problem, and certified-zero attributions for all lags beyond $K^*$, not merely small attributions, but mathematically zero within approximation error.

    \item \textit{Pillar~3 (Five Level Global Explanation Derivation).}
Given the estimated kernel, KARMA
extracts five layers of explanation without additional oracle queries.
\textbf{Level~1} computes the normalized variable importance or feature-level importance, ranking source variables by their total
distributional influence on the model's predictions across all
target variables and lags.
\textbf{Level~2} resolves this influence by lag $k$, yielding cell-level explanations (for example, influence on forecast at time $t$ of value at time $t-k$ for feature $d$)  that reveals whether the model has learned
momentum, mean-reversion, or long-memory dynamics for each variable.
\textbf{Level~3} identifies distinctive regimes via the marginal
interdependence index identifying distinct paths (histories) relevant to the forecast.
\textbf{Level~4} computes average interventional effects and the edge contributions that populate the model-induced causal graph,
connecting the explanation directly to the level 1.
\textbf{Level~5} quantifies explanation reliability: aleatoric
entropy measures genuine model uncertainty at each history,
while the epistemic variance flags histories where
kernel estimates are unreliable.

\end{enumerate}

\section{Related Work}
\label{sec:related}

\textbf{XAI for Time Series Forecasting.} Existing methods fall into three families.
\emph{Gradient-based} methods propagate prediction sensitivity back to inputs:
Saliency~\cite{simonyan2014}, Grad-CAM~\cite{selvaraju2017}, DeepLIFT~\cite{shrikumar2017},
Layer-wise Relevance Propagation~\cite{bach2015}, and Integrated Gradients~\cite{sundararajan2017},
often smoothed via SmoothGrad~\cite{smilkov2017}. These are fast but provide no
probabilistic guarantees and are sensitive to architecture; moreover, standard saliency
methods transfer poorly to temporal data, as benchmarked by Ismail et al.~\cite{ismail2020},
who propose Temporal Saliency Rescaling (TSR) to recover time-localized importance.
Temporal Integrated Gradients (TIG;~\cite{enguehard2023}) extends this family to sequential
settings but remains local.
\emph{Perturbation-based} methods (LIME~\cite{ribeiro2016}, SHAP~\cite{lundberg2017})
struggle with temporal autocorrelation since random perturbations break sequential
structure and induce off-manifold inputs. The simplest instance, Feature Occlusion
(FO;~\cite{suresh2017}), replaces a feature or group with a baseline and scores the
resulting output change; its augmented variant (AFO;~\cite{tonekaboni2020}) perturbs with
in-distribution samples to mitigate off-manifold artifacts. TimeSHAP~\cite{bento2021}
extends SHAP to recurrent models but inherits the baseline selection problem and provides
no certified attributions. FIT~\cite{tonekaboni2020} instead scores each observation by its
contribution to the predictive-distribution shift under a KL divergence against a
counterfactual where remaining features are unobserved, explicitly controlling for
time-dependent shift, but remains purely local. Dynamask~\cite{crabbe2021} learns a salient
input mask but operates locally with no reliability guarantees, and subsequent mask-based
refinements such as ContraLSP~\cite{liu2024} (contrastive, locally sparse perturbations) and
the surrogate explainer TimeX~\cite{queen2023} improve fidelity yet likewise yield only
instance-level, uncertified attributions. WinIT~\cite{leung2023} measures information
transfer across time windows but yields no certified attributions.
\emph{Attention-based} methods treat attention weights as explanations, a practice whose
reliability is contested~\cite{jain2019,wiegreffe2019}; interpretable forecasters such as
the Temporal Fusion Transformer~\cite{lim2021} expose variable-selection and attention
weights intrinsically but offer no statistical guarantees on the resulting attributions.
KARMA is a global surrogate method that explicitly models sequential dependence through
Markov structure, yielding explanations that respect the temporal geometry of the data and
carry explicit statistical reliability guarantees absent from all prior work.

 \textbf{Total Variation for Model Selection.} The formal definition of the total variation commonly used,
\begin{definition}[Total Variation Distance]
Let $\mathcal{M}([N])$ denote the set of conditional probability kernels on a set $[N]$. For $p(\cdot|h), q(\cdot|h) \in \mathcal{M}([N])$, $h$ in some history space, the \emph{total variation distance} is given by
\begin{equation}
    \|p(\cdot|h) - q(\cdot|h)\|_{TV} = \frac{1}{2}\sum_{s \in [N]} |p(s|h) - q(s|h)|
\end{equation}
\end{definition}
The use of total variation distance as a
Markov order stopping rule is related to \cite{csiszar2000} on consistent order
estimation, though our formulation targets surrogate equivalence rather than data
compression, and operates on the model kernel rather than the data kernel.

\section{K-Order Markov Chain Surrogate Model}

Building on the intuition introduced in the introduction, we now formalize KARMA as a K-order Markov surrogate model.
We model a multivariate time series (MVTS) as a stochastic process  $\mathbf X = \{X_t\}_{t \in \mathbb{Z}}, \: X_t = (X^1_t,\ldots,X^D_t)^\top \in \mathbb{R}^D$. Condider the predictor $f : \mathbb{R}^{W \times D} \to \mathcal{Y} \in \mathbb R^{T \times D}$
be a trained black-box predictor that takes an input window of length $W$ and outputs a $ T$-length prediction window. For the rest of this manuscript, we consider $T = 1$, $T>1$ follows by constructive extension.

Our goal is to construct an explanation based on the transition kernel of a surrogate $K$-order Markov chain that is (1) human–interpretable, (2) faithful to $f$ in the total variation sense, and (3) explicit about its own statistical reliability.

This section develops the construction of this surrogate model.

\begin{definition}[$K-$Order Markov chain]

Let $\mathbf X = (X_n)_{n \geq 1}$ be a stochastic process with
countable state space $\mathcal S$.
The process is a \emph{$k$-th order Markov chain} if, for all $n > k$, $X_n$ is independent of all other past time steps given the last $k$ steps.
The chain is \emph{time-homogeneous} if the transition probabilities
do not depend on $n$.

\end{definition}

\subsection{State Space Construction}

Denote $[D] = \{1, \dots, D\}$ and $[N] = \{1, \dots, N\}$. For each variable $d \in [D]$, define a measurable partition of $\mathbb{R}$
into $N$ bins using quantile boundaries estimated on the training data:

\[
\psi^d : \mathbb{R} \to [N], \qquad
\psi^d(x) = n \iff x \in [q^d_{n-1}, q^d_n).
\]

The multivariate discretization

\[
\psi : \mathbb{R}^D \to [N]^D
\]

induces the discrete state space $\mathcal S = [N]^D$ with $|\mathcal S| = N^D.$

The discretized process is $\tilde X_t = \psi(X_t) \in \mathcal S$. And for histories of length $K$, the corresponding history space is

\[
\mathcal H_K = \mathcal S^K, \qquad
|\mathcal H_K| = N^{DK}.
\]

where $h^d_k \in [N]$ is the discretized bin of variable $d$
at lag position $k$,
and the total number of distinct indices is $N^{DK} = |\mathcal{H}_K|$.

\subsection{The K-Order Markov Surrogate}

Let $f: \mathbb{R}^{W \times D} \to \mathcal{Y}$ be a trained black-box predictor 
and $\psi: \mathbb{R}^D \to \mathcal{S}$ the discretization map defined above. 
Since $f$ operates on a window of length $W$, its predictions may in principle 
depend on the full input history. However, the predictive information relevant 
to $f$ may be concentrated in a strictly shorter suffix. To identify the minimal 
such suffix, we prepend a baseline $b \in \mathcal{B}$ to a $K$-length suffix to 
form a complete input window and measure
\begin{equation}
    \hat{\Delta}_{\mathrm{pred}}(K, b) = \frac{1}{n}\sum_{i=1}^{n} 
    \ell\!\left(f(\tilde{h}_i),\, f(b \oplus h_i)\right),
    \label{eq:equiv}
\end{equation}
where $\tilde h_i \in \mathbb{R}^{W \times D}$ denotes the $W$ length $i$ observed history, $b \oplus h_i$ denotes the $K$ length suffix $h_i$ of $\tilde h_i$ concatenated with a $W - K$ length baseline $b$, and $\ell(\cdot,\cdot)$ a user imposed predictive difference metric (absolute difference for forecasting). 
$\hat{\Delta}_{\mathrm{pred}}(K, b)$ defines the average discrepancy between $f$'s predictions on the true window and on the 
substituted window. The selected order $K^*$ is the smallest $K$ for which 
$\min_{b \in \mathcal{B}}\,\hat{\Delta}_{\mathrm{pred}}(K, b) < \varepsilon$, $\varepsilon>0$ 
with $b^* = \arg\min_{b}\,\hat{\Delta}_{\mathrm{pred}}(K^*, b)$. When $K^* < W$, 
the model is provably insensitive to all inputs beyond lag $K^*$, yielding 
compression ratio $W/K^*$ and certified-zero attributions for all $k > K^*$. 
Importantly, $\hat{\Delta}_{\mathrm{pred}}$ requires only direct model queries, 
no kernel estimation, making it the sole stopping certificate.

Given $K^*$, the \textbf{surrogate transition kernel} of $f$ is
\begin{equation}
    \mathcal T^{f}_{K^*}(s \mid h) \;=\; 
    \mathbb P\!\left(\psi(f(\tilde{h})) = s \;\middle|\; 
    \mathrm{suffix}_{K^*}(\psi(\tilde{h})) = h\right), 
    \label{eq:joint_kernel}
\end{equation}

$$s \in \mathcal{S},\; h \in \mathcal{H}_{K^*}$$
with marginal at dimension $d \in [D]$
\begin{equation}
    \mathcal T^{f,d}_{K^*}(s^d \mid h) \;=\; 
    \mathbb P\!\left(\psi^d(f(\tilde{h})) = s^d \;\middle|\; 
    \mathrm{suffix}_{K^*}(\psi(\tilde{h})) = h\right).
    \label{eq:marginal_kernel}
\end{equation}
where $s^d$ is the d-th component of $s$. The $K$-order Markov chain associated with $\mathcal T^{f}_{K^*}$ is the 
\textbf{$K$-order Markov Surrogate} $\mathcal{M}_{K^*}$. Crucially, $f$ itself 
need not satisfy any Markov property: $\mathcal T^{f}_{K^*}$ defines the nearest 
$K$-order Markov approximation to $f$'s predictive behavior, not an extraction 
of latent Markov structure from the model.

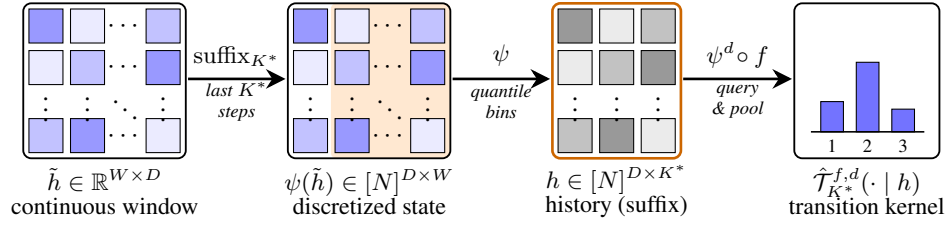
\begin{figure}
    \centering
   \begin{tikzpicture}[
  x=1cm,y=1cm,
  cont/.style ={draw,line width=0.4pt,minimum size=4.5mm,inner sep=0pt},
  disc/.style ={draw,line width=0.4pt,minimum size=4.5mm,inner sep=0pt},
  box/.style  ={rounded corners=3pt,line width=0.7pt},
  flow/.style ={-{Stealth[length=2.6mm,width=2.6mm]},line width=1pt},
  lbl/.style  ={font=\small,align=center},
  sub/.style  ={font=\scriptsize\itshape,align=center},
]
 
\def\cols{0,0.55,1.55}      
\def\rows{0,-0.55,-1.45}    
\def\cd{1.05}\def\rd{-1.0}  
 
\def\xA{0}\def\xB{3.5}\def\xC{7.0}\def\xD{10.1}
 
\foreach \x [count=\ci] in \cols{
  \foreach \y [count=\cj] in \rows{
    \pgfmathtruncatemacro{\m}{mod(\ci+\cj,3)}
    \ifcase\m \def\sh{8}\or\def\sh{24}\or\def\sh{40}\fi
    \node[cont,fill=blue!\sh] at (\xA+\x,\y) {};}}
\foreach \y in \rows{ \node at (\xA+\cd,\y) {$\cdots$}; }
\foreach \x in \cols{ \node at (\xA+\x,\rd) {$\vdots$}; }
\node at (\xA+\cd,\rd) {$\ddots$};
\draw[box] (\xA-0.30,0.30) rectangle (\xA+1.85,-1.75);
\node[lbl] at (\xA+0.78,-2.18) {$\tilde h\in\mathbb{R}^{W\times D}$\\[-1pt]continuous window};
 
\fill[orange!18,rounded corners=2pt] (\xB+0.27,0.30) rectangle (\xB+1.85,-1.75); 
\foreach \x [count=\ci] in \cols{
  \foreach \y [count=\cj] in \rows{
    \pgfmathtruncatemacro{\m}{mod(\ci+\cj,3)}
    \ifcase\m \def\sh{8}\or\def\sh{24}\or\def\sh{40}\fi
    \node[disc,fill=blue!\sh] at (\xB+\x,\y) {};}}
\foreach \y in \rows{ \node at (\xB+\cd,\y) {$\cdots$}; }
\foreach \x in \cols{ \node at (\xB+\x,\rd) {$\vdots$}; }
\node at (\xB+\cd,\rd) {$\ddots$};
\draw[box] (\xB-0.30,0.30) rectangle (\xB+1.85,-1.75);
\node[lbl] at (\xB+0.78,-2.18) {$\psi(\tilde h)\in[N]^{D\times W}$\\[-1pt]discretized state};
 
\def\hcols{0,0.55,1.10}
\foreach \x [count=\ci] in \hcols{
  \foreach \y [count=\cj] in \rows{
    \pgfmathtruncatemacro{\m}{mod(\ci+\cj,3)}
    \ifcase\m \def\sh{8}\or\def\sh{24}\or\def\sh{40}\fi
    \node[disc,fill=black!\sh] at (\xC+\x,\y) {};}}
\foreach \x in \hcols{ \node at (\xC+\x,\rd) {$\vdots$}; }
\draw[box,draw=orange!80!black,line width=1pt] (\xC-0.30,0.30) rectangle (\xC+1.40,-1.75);
\node[lbl] at (\xC+0.55,-2.18) {$h\in[N]^{D\times K^*}$\\[-1pt]history (suffix)};
 
\def\base{-1.40}
\foreach \bx/\bh [count=\bi] in {0.15/0.40, 0.62/0.92, 1.09/0.30}{
  \fill[blue!55,draw=black,line width=0.4pt]
    (\xD+\bx,\base) rectangle (\xD+\bx+0.30,\base+\bh);}
\draw[line width=0.6pt] (\xD+0.02,\base) -- (\xD+1.52,\base);
\foreach \bx/\bn [count=\bi] in {0.30/1, 0.77/2, 1.24/3}{
  \node[font=\scriptsize] at (\xD+\bx,\base-0.18) {\bn};}
\draw[box] (\xD-0.20,0.30) rectangle (\xD+1.72,-1.75);
\node[lbl] at (\xD+0.76,-2.18) {$\hat{\mathcal T}^{f,d}_{K^*}(\cdot\mid h)$\\[-1pt]transition kernel};
 
\draw[flow] (\xA+1.88,-0.70) -- (\xB-0.33,-0.70);
\node[lbl] at ($(\xA+1.88,-0.70)!0.5!(\xB-0.33,-0.70)+(0,0.30)$) {$\mathrm{suffix}_{K^*}$};
\node[sub] at ($(\xA+1.88,-0.70)!0.5!(\xB-0.33,-0.70)+(0,-0.30)$) {last $K^*$\\steps};
 
\draw[flow] (\xB+1.88,-0.70) -- (\xC-0.33,-0.70);
\node[lbl] at ($(\xB+1.88,-0.70)!0.5!(\xC-0.33,-0.70)+(0,0.32)$) {$\psi$};
\node[sub] at ($(\xB+1.88,-0.70)!0.5!(\xC-0.33,-0.70)+(0,-0.30)$) {quantile\\bins};
 
\draw[flow] (\xC+1.43,-0.70) -- (\xD-0.23,-0.70);
\node[lbl] at ($(\xC+1.43,-0.70)!0.5!(\xD-0.23,-0.70)+(0,0.32)$) {$\psi^{d}\!\circ f$};
\node[sub] at ($(\xC+1.43,-0.70)!0.5!(\xD-0.23,-0.70)+(0,-0.30)$) {query\\\&\ pool};
 
\end{tikzpicture}
\caption{\textbf{KARMA surrogate construction.} A continuous input window
$\tilde h\in\mathbb{R}^{W\times D}$ is discretized by the quantile map $\psi$ into
$[N]^{D\times W}$; its length-$K^*$ suffix forms the history $h\in[N]^{D\times K^*}$.
The $K^*$-order surrogate kernel $\hat{\mathcal T}^{f,d}_{K^*}(\cdot\mid h)$ is the
distribution of the discretized model output $\psi^d(f(\tilde h))$ over next-bins
$s^d\in[N]$, conditioned on $h$ and pooled over all windows landing in that history.}
\label{fig:karma-surrogate}
\end{figure}

\section{Why Transition Probabilities Are (Global) Explanations}
\label{sec:why}

Standard attribution methods ask: \textit{how much does feature $x_i$ contribute 
relative to a baseline?} For time series, this framing is problematic --- 
$X_{t-2}$ and $X_{t-1}$ are strongly dependent, so marginalizing one while 
fixing the other produces counterfactuals off the data manifold, and 
gradient-based explanations describe local sensitivity at a single point rather 
than the global conditional structure the model has learned. Transition 
probabilities answer a different and more natural question: \textit{given that 
the system was in state $h$ for the last $K^*$ steps, what does the model 
predict next?} For a $K^*$-order Markov process, $\mathcal T^{f}_{K^*}$ and 
$\mathcal T^{f,d}_{K^*}$ are sufficient statistics for next-state prediction, so under 
output equivalence at tolerance $\varepsilon$ the surrogate kernel captures 
precisely the predictive dependencies the model has learned between variables 
and lags. Under the faithfulness assumption (in the sense of causality), this directly reveals which variables at which lags act as direct causal drivers. 
\subsection{The Five-Level Explanation Hierarchy}
\label{subsec:5-level}
Let $K^* >0$, $h \in \mathcal H_{K^*}$, $d \in [D]$ and $f$ a trained model. Further assume $\hat{\mathcal T}_{K^*}^{f,d}$ the approximation of $\mathcal T_{K^*}^{f,d}$ such that 
\begin{equation}
\label{eq:rho_floor_h}
    \mathbb{E}\!\left[
    \|\hat{\mathcal T}^{f,d}_K(\cdot \mid h)-\;
               \mathcal T^{f,d}_K(\cdot \mid h)\|_{TV}
  \right] < \frac{\lambda}{4}, \qquad \lambda > 0
\end{equation}
 We demonstrate how marginal transition probabilities are used to provide global explanations. We also go on to prove in 
Supplementary Material~\ref{app:dag} that under a 
model-centric faithfulness condition on $\mathcal M_{K^*}$, the resulting attributions are causally grounded in a 
precise sense: nonzero attributions correspond exactly to 
direct effects in the model-induced causal graph, and zero 
attributions to d-separated pairs.

\subsubsection{Level 1: Variable Importance - Who Matters?}
\begin{definition}[Lag-Resolved Influence and Variable Importance]
    For any dimension $d' \in [D]$ at lag $k \in \{1, \dots, K^*\}$, the \emph{lag-resolved influence} on the forecast at time $t$ is 
    \begin{equation}
    \label{eq:lag-resolved-influence}
        \phi_k^{d'} = \sum_{d = 1}^D \rho(X_{t-k}^{d'} \rightarrow X_{t}^{d}) \cdot \mathrm{1}\{\rho(X_{t-k}^{d'} \rightarrow X_{t}^{d}) > \lambda\}.
    \end{equation}
 where:
\begin{equation}
  \label{eq:rho}
  \rho(X^{d'}_{t-k} \to X^{d}_{t}) = \sum_{h \in \mathcal{H}_{K^*}} \hat{\pi}^*(h)\cdot
  \frac{1}{2} \Delta_{TV},
\end{equation}
with $$\Delta_{TV} = \sum_{s^{d} \in [N]}
  \left|\hat{\mathcal T}^{f,d}_{K^*}(s^{d}\mid h)
        - \frac{1}{N}\sum_{x=0}^{N-1}
          \hat{\mathcal T}^{f,d}_{K^*}(s^{d}\mid h^{d' \leftarrow x}_k)\right|,$$

where $h_k^{d'\leftarrow x} $ replaces only the $d'$-th component at lag $k$.
 $\hat{\pi}^*(h) = \mathbb{P}(\psi(X_{t-K^*+1:t}) = h) > 0$
is the approximate stationary probability of history $h$. Intuitively, $\rho(X^{d'}_{t-k} \to X^d_t)$ measures how much the model's
forecast distribution for variable $d$ shifts when we average out the influence
of variable $d'$ at lag $k$. $\lambda$ here serves as a threshold for which there exists genuine causal dependency, not noise arising from approximations, and prevents false positive and false negative causal edges with probability dependent on the upper bound estimate of $\mathbb{E}\!\left[
    \|\hat{\mathcal T}^{f,d}_K(\cdot \mid h)-\;
               \mathcal T^{f,d}_K(\cdot \mid h)\|_{TV}
  \right]$ ($\lambda \le 0.1)$ is recommended, Supplementary Material~\ref{app:lambda}).

     The \emph{Marginal Total-Variation Influence} of variable $d'$ aggregated across all lags is:
    \begin{equation}
    \label{eq:marginal-tv-influence}
        \Phi^{d'} = \sum_{k=1}^{K^*} \phi_k^{d'}
    \end{equation}
     and finally define the \emph{Variable Importance} as the normalization of the marginal total-variation influence.
 The quantity $\tilde\Phi^{d'}$ quantifies the total influence of $\{X^{d'}\}_{t-k}^{t-1}$ to predict $\mathbf X_t$ for any $t \ge 0$.
\end{definition}
This ranking is (i) global, averaged over the full state space; (ii) probabilistically grounded, measuring actual distributional shift rather than contribution to a local linear approximation; and
(iii) directly comparable across variables, as all values lie in $[0, 1]$.

\subsubsection{Level 2: Lag Profiles - When does Each Variable Matter?}
The lag-resolved influence in equation \eqref{eq:lag-resolved-influence} plotted against $k$ reveals the temporal shape of
how the model uses past information. The decay profile reveals whether the model has learned momentum, volatility clustering, or mean-reversion.



\subsubsection{Level 3: Marginal Regime Explanation (Top-$k$) - What Sequences Matter?}
For target dimension $d \in [D]$ and history sequence $h \in \mathcal{H}_{K^*}$, we define
\textit{marginal interdependence index} as:

\begin{align*}
    \Psi^d(h) &= \mathbb{E}_{h' \sim \hat{\pi}^*}
    \bigl[\|\hat{\mathcal{T}}^{f,d}_{K^*}(\cdot\mid h)-\;
                     \hat{\mathcal{T}}^{f,d}_{K^*}(\cdot\mid h')\|_{TV}]\\
            &= \sum_{h' \in \mathcal{H}_{K^*}} \hat{\pi}^*(h') \cdot
    \|\hat{\mathcal{T}}^{f,d}_{K^*}(\cdot\mid h)-\;
                     \hat{\mathcal{T}}^{f,d}_{K^*}(\cdot\mid h')\|_{TV}
\end{align*}
the $\hat{\pi}^*$-weighted average total variation distance between the model's
prediction at $h$ and its prediction at a history drawn from
the stationary distribution.
A high value of $\Psi^d(h)$ identifies $h$ as a distinctive
regime: the model behaves unusually at $h$ relative to its
typical conditional behavior across the history space. Level~3 reports the score table $\{\Psi^d(h)\}_{d \in [D],\,
h \in \mathcal{H}_{K^*}}$.

\subsubsection{Level 4: Interventional Profiles - What Happens When We Fix a Variable?}
For $h \in H^+_{K^*}$, the \emph{counterfactual history}
$h^{d'\leftarrow x}_k$ replaces the $d'$-th component of $h$
at lag $k$ with bin $x$. The \emph{average interventional
effect} (AIE) of variable $d'$ at lag $k$ on variable $d$ is

\begin{equation}
    \mathrm{AIE}_d(d',k,x) = \sum_{h}\hat{\pi}^*(h)\left \| \hat{\mathcal T}^{f,d}_{K^*}(\cdot\mid h)
        - \hat{\mathcal T}^{f,d}_{K^*}(\cdot\mid h^{d' \leftarrow x}_k) \right\|
    \label{eq:AIE}
\end{equation}
$\mathrm{AIE}_d(d',k,x)$ measures $f$'s
predictive sensitivity to fixing past inputs; it is a property
of the surrogate $\mathcal  M_{K^*}$, not of the data-generating process.
One can verify that $\frac{1}{N}\sum_x
\mathrm{AIE}_d(d',k,x) = \rho(X^{d'}_{t-k}\to X^d_t)$,
connecting Level 4 directly to the edge-trimming criterion of
Level 1. Level 4 reports the ranked table
$\{\frac{1}{N}\sum_x\mathrm{AIE}_d(d',k,x)\}_{d,d',k}$.
\subsubsection{Level 5: Uncertainty of Explanations}
\textit{Per-variable aleatoric uncertainty} is given by the entropy of the predictive distribution:
\begin{equation}
\label{eq:aleatoric uncertainty}
    H^d_h = -\sum_{s^{ d}} \hat{\mathcal{T}}^{f,d}_{K^*}(s^{ d}\mid h)
\log \hat{\mathcal{T}}^{f,d}_{K^*}(s^{d}\mid h),
\end{equation}

 We estimate the \textit{Epistemic Uncertainty} via an MLE bootstrap variance:

\begin{equation}
    \mathrm{Var}(\hat{\mathcal{T}}^{f,d}_{K^*}(s^{ d}\mid h))
\approx \hat{\mathcal{T}}^{f,d}_{K^*}(s^{ d}\mid h)\left[1 - \hat{\mathcal{T}}^{f,d}_{K^*}(s^{ d}\mid h)\right] / n^f(h).
\end{equation}

where $n^f(h)$ is the query-visit count for history h. Histories with low $n^f(h)$ flag unreliable explanations,
yielding a natural \textit{explanation coverage }map. No other XAI method for time series provides this
reliability certificate.

\begin{figure*}[!ht]
  \centering
  \begin{minipage}[t]{0.32\textwidth}
    \centering
    \includegraphics[width=\textwidth]{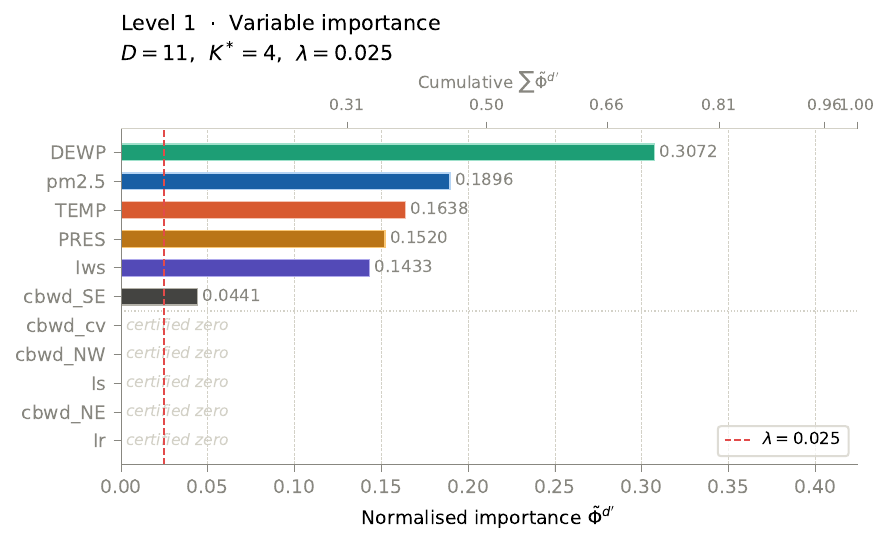}
    \centerline{(a)}
  \end{minipage}
  \hfill
  \begin{minipage}[t]{0.32\textwidth}
    \centering
    \includegraphics[width=\textwidth]{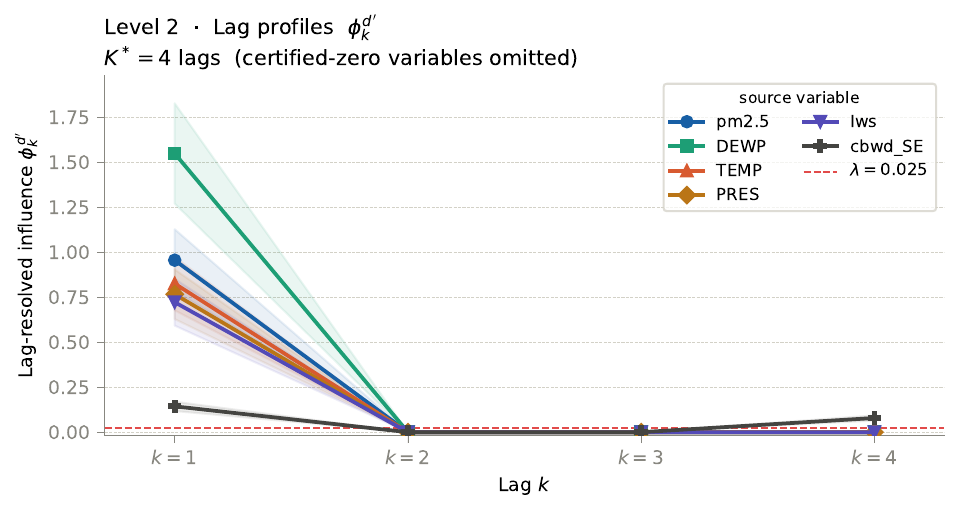}
    \centerline{(b)}
  \end{minipage}
  \hfill
  \begin{minipage}[t]{0.32\textwidth}
    \centering
    \includegraphics[width=\textwidth]{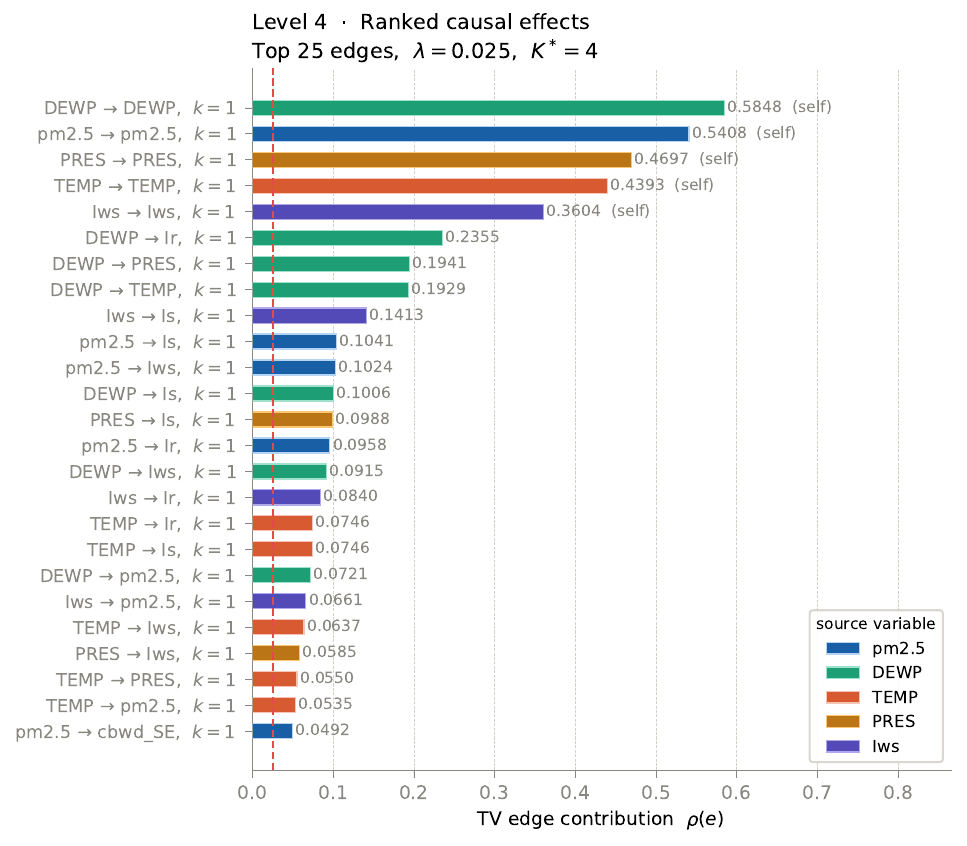}
    \centerline{(c)}
  \end{minipage}

  \vspace{1.0em}

  \begin{minipage}[t]{0.32\textwidth}
    \centering
    \includegraphics[width=\textwidth]{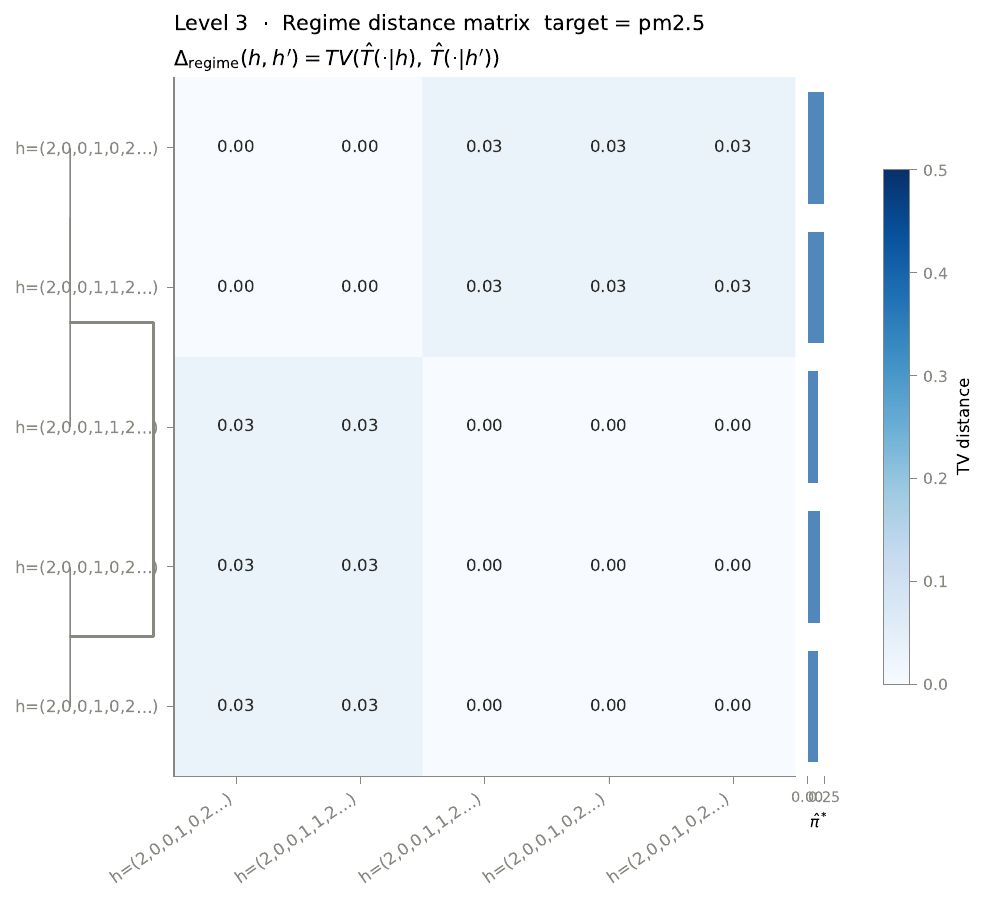}
    \centerline{(d)}
  \end{minipage}
  \hfill
  \begin{minipage}[t]{0.60\textwidth}
    \centering
    \includegraphics[width=\textwidth]{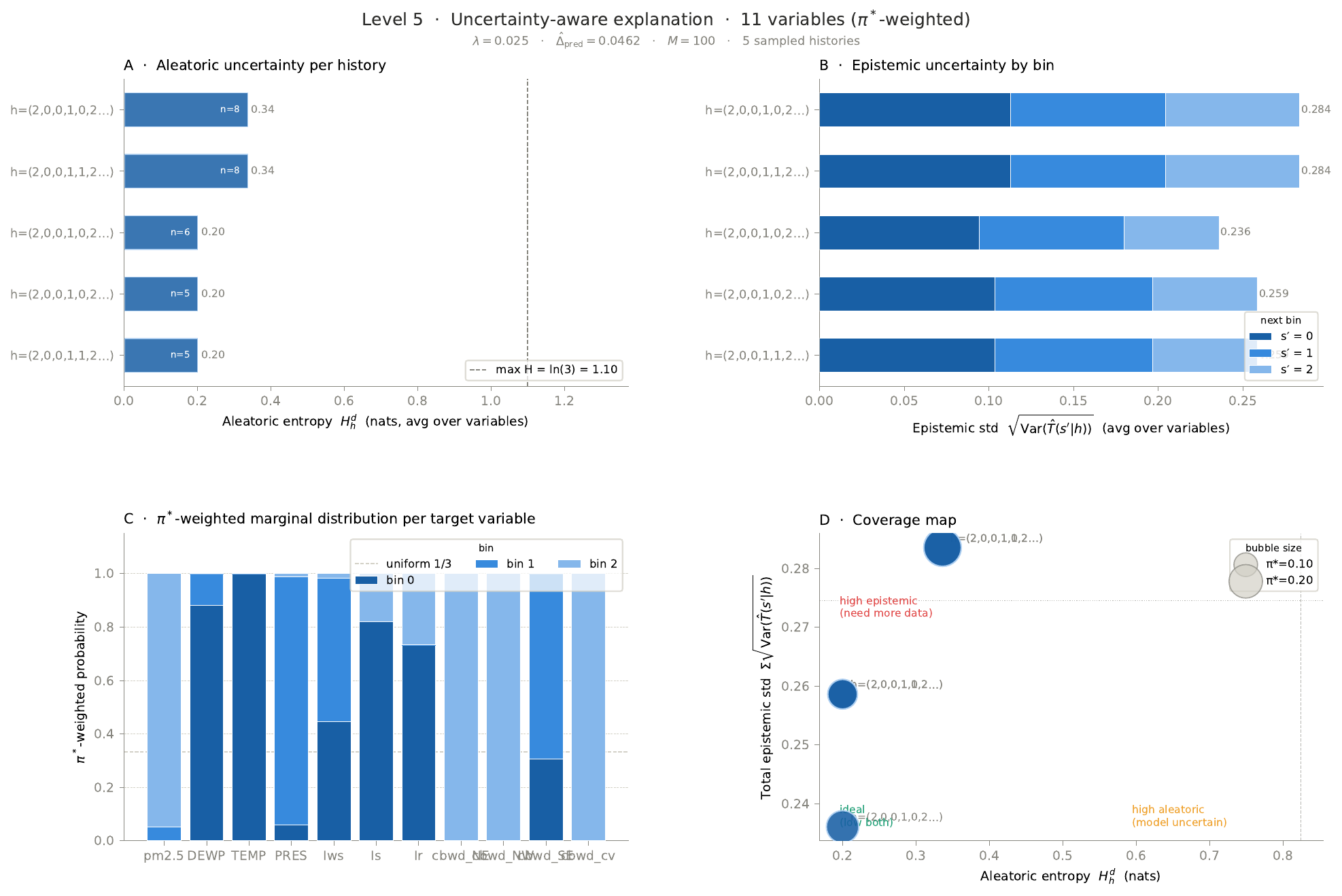}
    \centerline{(e)}
  \end{minipage}
  \hfill
  \begin{minipage}[t]{0.32\textwidth}
    \centering
  \end{minipage}

  \caption{
  \textbf{Five-level \textsc{KARMA} explanation hierarchy} applied to a
  TCN ($W=24$, hidden~64, 2~layers, kernel~3) trained on the Beijing
  PM$_{2.5}$ dataset ($D=11$, $N=3$, $K^*=4$, $\varepsilon = 0.049$,
  $\lambda=0.025$, $\hat{\Delta}^{\mathrm{pred}}=0.046$).
  All levels are derived from the single estimated kernel
  $\hat{\mathcal T}^{f,d}_{K^*}$ without additional oracle queries.
  History indices $h$ follow the mixed-radix encoding
  $h = \sum_{k,d} h^d_k \cdot N^{D(K^*-1-k)+(D-1-d)}$.
  \textbf{(a) Level~1 --- Variable importance} $\tilde{\Phi}^{d'}$:
    \textsc{dewp} dominates; the remaining wind-direction and precipitation
    variables receive certified-zero attributions ($<\lambda$).
  \textbf{(b) Level~2 --- Lag profiles} $\phi^{d'}_k$:
    active variables peak sharply at $k=1$ and decay over the
    $K^*=4$ retained lags, with \textsc{temp} retaining secondary
    influence (certified-zero variables omitted).
  \textbf{(c) Level~4 --- Ranked interventional effects} $\rho(e)$:
    self-loops dominate the top-25 retained edges; the strongest
    cross-variable effect is \textsc{dewp}$\to$\textsc{temp}.
  \textbf{(d) Level~3 --- Regime distance matrix} (target \textsc{pm2.5}):
    pairwise total-variation distances
    $\|\hat{\mathcal T}^{f,d}_{K^*}(\cdot\mid h)-\hat{\mathcal T}^{f,d}_{K^*}(\cdot\mid h')\|_{TV}$
    between histories, hierarchically clustered; bright off-diagonal blocks
    mark distinctive regimes where the model's conditional forecast departs
    from its typical behaviour, the pairwise view from which
    $\Psi^d(h)$ aggregates.
  \textbf{(e) Level~5 --- Uncertainty-aware explanation}
    ($\hat{\pi}^*$-weighted):
    \textbf{(A)} per-history aleatoric entropy $H^d_h$ (bits, averaged over
    variables);
    \textbf{(B)} epistemic uncertainty by next-bin, the bootstrap standard
    deviation $\sqrt{\mathrm{Var}(\hat{\mathcal T}^{f,d}_{K^*})}$;
    \textbf{(C)} $\hat{\pi}^*$-weighted marginal next-bin distribution per
    target variable, against the uniform $1/N$ reference;
    \textbf{(D)} coverage map of aleatoric entropy vs.\ epistemic deviation
    with bubble area $\propto$ visit count $n^f(h)$, flagging high-epistemic,
    low-coverage histories that warrant more data.}
  \label{fig:KARMA-hierarchy}
\end{figure*}

\section{KARMA}
\label{sec:framework}
The complete KARMA procedure is summarized in Algorithm~\ref{alg:KARMA} 
(Supplementary Material~\ref{app:algorithm}); we describe each step in turn.
\subsection{Step 1: State Space Discretization}

Each marginal series $X^d$ is discretized into $N$ bins via quantile boundaries on the training window. We recommend $N \in \{2,3,4,5\}$.

\subsection{Step 2: $K^*$ Selection and Certified Attribution}
\label{sec:step2}

\paragraph{Surrogate validity.}
The selected order $K^*$ is the smallest $K$ for which the model is
certified insensitive to the prefix:
\begin{equation}
  \label{eq:kstar}
  K^* = \min\!\left\{K \geq 1 :
    \min_{b \in \mathcal{B}}\hat{\Delta}^{\mathrm{pred}}(K, b) < \varepsilon\right\}, \:\:\: \varepsilon>0,
\end{equation}

$\hat{\Delta}^{\mathrm{pred}}$ is a \textbf{direct model test} requiring no
kernel estimation: query $f$ twice per sample, once with the full window and
once with $b^* \oplus h_i^{}$, and average the prediction discrepancy.
\footnote{Optimally, for maximum confidence in the surrogacy approximation, it is performed on unseen data to ensure transparency. However, due to data scarcity, the validation split can be used as an alternative.}

\subsection{Step 3: KARMA Baseline and Model Compression}
\label{sec:baseline}

When $K^* < W$, define the \textbf{KARMA baseline} $b^*$ as the prefix
minimizing the average prediction discrepancy when the true prefix is replaced:
\begin{equation}
    \label{eq:bstar}
    b^* = \argmin_{b}\hat{\Delta}^{\mathrm{pred}}(K^*, b).
\end{equation}
Unlike other baselines, which are analyst-chosen, $b^*$ is
\textbf{model-certified}: it is the prefix that makes $f$'s predictions
under the surrogate closest to $f$'s true predictions, discovered from
model behavior. Define the \textbf{Compression ratio} by $W/K^*$.\footnote{Model compression and the certified-zero result
require only that $\hat{\Delta}^{\mathrm{pred}}$ be estimated reliably, which
demands $n \gtrsim 200$ independent model queries for stable averages,
easily satisfied with any held-out series of length $T \geq W + 200$.
Pillars~1--2 are therefore applicable under standard data budgets without
qualification.
}

\subsection{Step 4: (Marginal) Transition Kernel Estimation}
 
The five-level explanations are computed from the marginal transition kernel
$\hat{\mathcal{T}}^{f,d}_{K^*}$. Estimating the full joint kernel is statistically
prohibitive: the number of histories grows as $N^{DK^*}$, so the data required
for a uniform total-variation guarantee is exponential in $DK^*$ and infeasible
at standard series lengths. KARMA therefore estimates $D$ factored marginal
kernels, one per target variable, which a single set of model queries supplies
simultaneously.
 
We provide three estimators, all with finite-sample total-variation guarantees; the choice trades data efficiency against
implementation simplicity.
 
\paragraph{Strategy A (empirical counting).} A Dirichlet posterior-mean
estimator accumulates discretised model outputs into per-history counts. It is
simple and unbiased in the large-sample limit, but its sufficient sample budget
still scales with $N^{DK^*}$, limiting it to small state spaces.
 
\paragraph{Strategy B ($b^*$-prefix Monte Carlo, default).} 
We provide first the following definition

\begin{definition}[Observed history support]
\label{def:support}
The \textbf{observed support} of $\hat{\pi}^*$ at order $K$ is:
\begin{equation}
  \mathcal{H}^+_K
  = \bigl\{h \in \mathcal{H}_K : \hat{\pi}^*(h) > \theta, \:\: \theta\ge0\bigr\},
  \label{eq:support}
\end{equation}
the set of histories that actually appear in the training data, where $\hat{\pi}^*(h) = \mathbb{P}(\psi(X_{t-K^*+1:t}) = h) > 0$.
All
$\hat{\pi}^*$-weighted explanation quantities are zero outside this set.
Restricting to $\mathcal{H}^+_{K^*}$ bounds the effective history space by 
$|\mathcal{H}^+_{K^*}| \leq T - K^*$, growing linearly in $T$ rather than 
exponentially in $DK^*$. We use $\theta = 0$ is the default setting.
\end{definition}

Rather than relying
on observed history counts, Strategy~B draws samples from a distribution pool, substitutes the certified baseline $b^*$ on the prefix,
and queries the model. Its sufficient budget scales with $|\mathcal{H}^+_{K^*}|
\leq T - K^*$, linear in the series length rather than exponential in
$DK^*$, which is what makes KARMA tractable in practice. We use Strategy~B by
default.

\paragraph{Strategy C (tree-structured pooling).}
Strategies~A and~B estimate each history's kernel independently, so they fail
on histories that are never sufficiently covered, no matter how long the series.
Strategy~C instead grows a regression tree that partitions the covered
histories $\mathcal{H}^{\mathrm{cov}}$ (patterns with sufficient coverage) into leaves of structurally similar
histories and pools their model counts, sharing statistical strength across
histories that agree on the bins the tree identifies as predictive. Sparse
histories inherit the pooled kernel of the leaf they route to, extending
reliable estimates to histories inaccessible to Strategies~A and~B. The sample
budget scales with the number of leaves $L \ll |\mathcal{H}^+_{K^*}|$ rather
than with $|\mathcal{H}^+_{K^*}|$, and when every history is covered the tree
reduces to one history per leaf and recovers Strategy~B exactly. The split
variables additionally read off as Level~1--2 structure: a source variable at a
given lag that appears as a high split node for target $d$ is a primary driver
of $d$.

All estimators admit explicit noise floors that upper-bound the per-history estimation error \eqref{eq:rho_floor_h} and feed directly into the Level~5 reliability map; the precise bounds and sample budgets are stated and proved in Supplementary Material.
 
\begin{remark}[Reliability below the sufficient budget]
The sample thresholds are \emph{sufficient} conditions for kernel error below
$\lambda/4$; reliable estimates are routinely obtained with fewer observations.
We therefore recommend reporting Level~5 alongside any KARMA output, as the
epistemic variance map certifies reliability under the actual data budget.
\end{remark}

\subsection{Step 5: 5-level Explanation Retrieval}
See Section~\ref{subsec:5-level}

\section{Experiments}
\label{sec:experiments}
We compare KARMA (KM) to three temporal aware methods, TimeSHAP (TS), WinIT (WI), and Dynamask (DM), the Feature Occlusion (FO) method, and the gradient-based method Integrated Gradients (IG). Code available on \href{https://github.com/AmTuTi1999/karma_.git}{https://github.com/AmTuTi1999/karma$\_$.git}.
\subsection{Case Study A: Synthetic Data}
\subsubsection{Experimental setup}
\label{sec:varmulti-setup}

All methods explain the same analytical oracle $f_{\mathrm{VAR}}$, which
implements the true VAR data-generating process exactly (no model error) with known true Markov Order $K_{true}$ given coefficient matrix $A \in \mathbb{R}^{D \times D \times K_{true}}$:
\[
    f_{\mathrm{VAR}}(x)_j
    \;=\; \sum_{k=1}^{K_{true}} \sum_{i=1}^{D} A_{ij}^{(k)}\, x_{W-k,\,i}, \qquad j \in [D]
\]
where $x \in \mathbb{R}^{W \times D}$ is the input window of length $W$.
The ground-truth importance matrix is
$\Phi^*_{i,k} = \sum_j |A_{ij}^{(k)}|$,
aggregating the absolute coefficient magnitudes over all target variables.
Methods are ranked by Kendall's $\tau$ between their flattened attribution
estimate $\hat{\Phi}$ and $\Phi^*$.

For KARMA, the Markov order $K^*$ is immediately obtained as $K_{true}$ and certified baseline $b^*$ is selected via
$\Delta^{\mathrm{pred}} < \varepsilon$ with $\varepsilon = 10^{-5}$.
The history discretizer uses $N = 3$ bins for $D \leq 4$ and $N = 2$
for $D > 4$ to keep the symbolic state space tractable. $M = 100$ Monte Carlo draws per history are used to build sampling pool for kernel estimation.

To obtain cell-level global explanations from FO, DynaMask, and TimeSHAP, the raw $(B, D, W)$ attribution
tensor is aggregated to $\hat{\Phi} \in \mathbb{R}^{D \times W}$ by
\[
    \hat{\Phi}_{d}
    \;=\; \frac{1}{B}\sum_{b=1}^{B} |\phi_{b,d,\,W-k}|, \qquad d \in [D], \: k \in [W]
\]
reading each lag $k$ from window position $W - k$.
WinIT uses the last-target slice at the same window positions.
KARMA produces $\hat{\Phi}$ directly without further aggregation.

We exclude IG from the synthetic comparison: on the linear VAR oracle, integrated gradients reduce to the exact coefficient magnitudes, making the comparison degenerate rather than informative. IG is retained in the real-world evaluation (Table~\ref{tab:lag_drop25_var}), where the models are nonlinear.
\subsubsection{Results}
\label{sec:synthetic-results}

We first evaluate \textsc{KARMA} on model-induced causal graph recovery,
treating $\rho(e) > \lambda$ (0.025 for \textsc{tiny}, 0.1 for the rest) as the criterion for a retained edge $e$. Table~\ref{tab:var_karma_edges} reports the precision, recall, and F1 scores of edge discovery. KARMA achieves a recall score of $1.0$ in all but the most complex case (still achieving 0.71) showing KARMA excludes no true causal relationships. The high precision and F1 indicates KARMA's ability to not retain spurious edges.

\begin{table}[H]
\centering
\caption{KARMA edge-recovery metrics, $|E|$ denotes the number of edges. Recall is 1.00 through \emph{large};
         precision and recall also stay consistently close or equal to 1.00.}
\label{tab:var_karma_edges}
\setlength{\tabcolsep}{6pt}
\begin{tabular}{lrrr ccc}
\toprule
Config & $D$ & $K_{true}$ & $|E|$ & Precision & Recall & F1 \\
\midrule
tiny   & 2 & 1 &  2 & 1.00 & 1.00 & 1.00 \\
small  & 4 & 2 &  5 & 1.00 & 1.00 & 1.00 \\
medium & 4 & 3 &  7 & 0.70 & 1.00 & 0.82 \\
large  & 6 & 3 &  9 & 1.00 & 1.00 & 1.00\\
xlarge & 8 & 4 & 14 & 1.00 & 0.71 & 0.83 \\
\bottomrule
\end{tabular}
\end{table}

We further evaluate across five VAR configurations of increasing scale
(Table~\ref{tab:var_tau}). \textsc{KARMA} is the sole top-ranked method on
medium, large, and xlarge configurations, where the other methods degrade
as the graph grows denser.
\begin{table}[H]
\centering
\caption{Kendall's $\tau$ vs.\ ground truth across VAR configurations.
         \textbf{Bold} marks the highest per row (ties included). $*$ indicates Kendall test with p value less than 0.05}
\label{tab:var_tau}
\setlength{\tabcolsep}{5pt}
\begin{tabular}{lrrr ccccc}
\toprule
Config & FO & DM & WI & TS & KARMA \\
\midrule
tiny  & $0.36$ & $\phantom{-}0.76^{\phantom{*}}$ & $\mathbf{0.94^*}$  & $\mathbf{0.94^*}$  & $\mathbf{0.94^*}$ \\
small  & $\phantom{-}0.13$ & $\phantom{-}0.14^{\phantom{*}}$ & $0.86^*$ & $\mathbf{0.92^*}$ & $\mathbf{0.92^*}$\\
medium & $\phantom{-}0.26$ & $-0.09^{\phantom{*}}$ & $0.77^*$ & $0.79^*$ & $\mathbf{0.90^*}$ \\
large  & $-0.07$ & $-0.32^*$ & $0.72^*$ & $0.68^* $& $\mathbf{0.92^*}$ \\
xlarge & $\phantom{-}0.25$ & $\phantom{-}0.09{\phantom{*}}$ & $0.76^*$ & $0.64^*$& $\mathbf{0.77^*}$ \\
\bottomrule
\end{tabular}
\end{table}

\subsection{Case Study B: Real-World Forecasting Benchmark}

We consider three diverse time series datasets: Electricity Transformer Temperature hourly 1 dataset (ETTh1) \cite{zhou2021informer}, Beijing
PM$_{2.5}$ (Bei) \cite{liang2015assessing}, Exchange Rate (ExRa) \cite{lai2018modeling}, and three time series forecasting models (GRU, LSTM, TCN). To compare KARMA to the different methods, we use a \textit{temporal aware lag based AUC score} ($\mathbf{AUC}_{lag}$) (Supplementary Material~\ref{sec:auc_lag}) to measure predictive change under time lag occlusion i.e. we evaluate how well the XAI method learns model induced temporal dependencies by replacing entire time horizons with conditional expectations learnt by a VAR on transition dynamics on the training data.

Given a batch of $B$ test windows and a $D$-variate attribution tensor
$\phi \in \mathbb{R}^{B \times D \times T}$, we aggregate attributions into a
per-timestep importance score
\begin{equation}
    s_t \;=\; \frac{1}{BD}\sum_{b=1}^{B}\sum_{d=1}^{D}\bigl|\phi_{b,d,t}\bigr|,
    \qquad t = 1,\ldots,T.
\end{equation}
For \textsc{KARMA}, we bypass the cell attribution matrix and derive time scores
directly from the raw edge weights,
\begin{equation}
    s_t^{\textsc{KARMA}} \;=\!\!\sum_{\{e\,:\,T - \ell_e = t\}}\!\!\rho_e,
\end{equation}
where $\ell_e$ is the lag and $\rho_e$ is the importance of edge $e$, so that
a timestep targeted by many high-weight causal edges scores proportionally higher.

We report $\mathbf{AUC}_{lag}\text{@}25\%$, a
single-point summary of how much prediction changes when the top quarter of
timesteps are removed.
A method that correctly identifies the temporal dependencies concentrates its
top-ranked positions there, producing a steep early rise and a high area;
one that wastes rankings on non-blanket lags sees a flat curve.

\subsubsection{Results.}
Table~\ref{tab:lag_drop25_var} reports AUC$_{\text{lag}}$@25\% across the three
datasets and three architectures, with the full removal curves shown in
Figure~\ref{fig:lag_curves}. On ETTh1, KARMA attains the highest score on
every architecture and opens a clear margin on the TCN ($0.71$ vs.\ $0.49$ for
the next-best method), indicating that its edge-derived timestep scores
concentrate prediction-relevant lags more sharply than the flat baselines. On
Beijing PM2.5, KARMA is competitive throughout, matching or marginally trailing
the strongest baseline on each architecture. On Exchange Rate, all methods
collapse to near-zero AUC$_{\text{lag}}$@25\%: this dataset is close to a random
walk, where successive increments carry little learnable temporal structure, so
no attribution method can identify timesteps whose removal
substantially shifts the forecast. We read this not as a failure of the
explainers but as the metric behaving correctly: when the model has learned
little exploitable temporal dependence, a faithful temporal-importance score
\emph{should} be flat. Notably, the strong perturbation baselines IG and
TimeSHAP, which are competitive on ETTh1, exhibit the same collapse on Exchange
Rate, confirming that the effect is a property of the data rather than of
KARMA's construction.

\begin{table}[t]
\centering
\caption{$\mathbf{AUC}_{lag}\text{@}25\%$ under VAR-conditional imputation. Higher is better.}
\label{tab:lag_drop25_var}
\setlength{\tabcolsep}{5pt}
\begin{tabular}{llcccccc}
\toprule
Dataset & Arch & FO & DM & WI & IG & TS & KM \\
\midrule
\multirow{3}{*}{ETTh1}
  & GRU  & 0.03 & 0.04 & 0.03 & 0.61          & 0.61          & \textbf{0.63} \\
  & LSTM & 0.16 & 0.08 & 0.09 & 0.62          & \textbf{0.66} & \textbf{0.66} \\
  & TCN  & 0.03 & 0.24 & 0.02 & 0.49          & 0.49          & \textbf{0.71} \\
\midrule
\multirow{3}{*}{ExRA}
  & GRU  & 0.00 & 0.01 & 0.00 & \textbf{0.03} & \textbf{0.03} & 0.02 \\
  & LSTM & 0.00 & \textbf{0.06} & 0.00 & 0.01 & 0.01          & 0.00 \\
  & TCN  & 0.00 & 0.01 & 0.00 & 0.02          & 0.02          & \textbf{0.05} \\
\midrule
\multirow{3}{*}{Bei}
  & GRU  & 0.01 & 0.01 & \textbf{0.29} & \textbf{0.29} & 0.22 & 0.18 \\
  & LSTM & 0.01 & 0.01 & 0.10          & \textbf{0.25} & 0.24 & 0.24 \\
  & TCN  & 0.02 & 0.01 & \textbf{0.28} & 0.22          & 0.19 & \textbf{0.28} \\
\bottomrule
\end{tabular}
\end{table}

\begin{figure}[t]
    \centering
    \includegraphics[width=\linewidth]{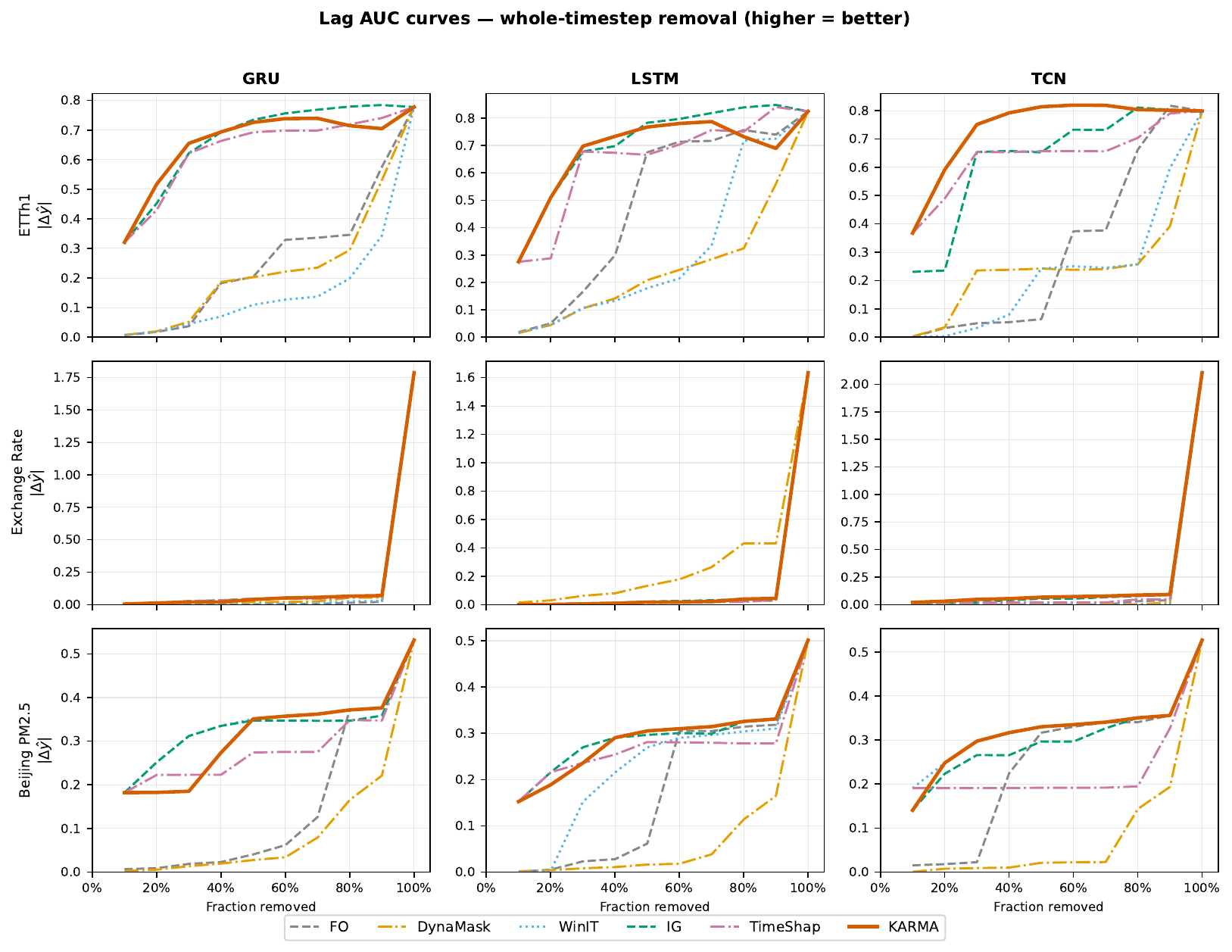}
    \caption{Lag AUC removal curves for all datasets and architectures.
             Each curve shows cumulative $|\Delta\hat{y}|$ as the fraction
             of removed timesteps grows from 0 to 1.}
    \label{fig:lag_curves}
\end{figure}


\section{Conclusion}

We presented KARMA, a framework for certified temporal explanation of
black-box models on multivariate time series via $K$-order Markov
approximations. KARMA makes three contributions. Pillar~1 selects the minimal
predictively sufficient lag $K^*$. Pillar~2 delivers a model compression
theorem certifying prediction fidelity loss $\leq \varepsilon$ at compression
ratio $W/K^*$, a model-certified baseline $b^*$, and certified-zero
attributions for all lags beyond $K^*$, not merely small attributions.
Pillar~3 extracts five layers of global explanation from the single estimated
kernel without additional oracle queries, spanning variable importance, lag
profiles, regime distinctiveness, average causal effects, and uncertainty
quantification.

On synthetic VAR benchmarks with known ground truth, KARMA recovers the
model-induced causal structure with perfect recall on all but the largest
configuration ($1.00$ through the \texttt{large} setting, $0.71$ on
\texttt{xlarge}), correctly retaining every true edge; its high precision on
all graphs reflects a robustness to spurious relationships that might arise from autocorrelation. On feature-importance ranking, KARMA is the sole top-ranked method on
the \texttt{medium}, \texttt{large}, and \texttt{xlarge} configurations, where
perturbation-based baselines such as TimeSHAP and WinIT degrade as graph
density grows, while remaining competitive on the smaller settings. On
real-world forecasting across three datasets and three architectures, KARMA
matches or exceeds the strongest baselines under VAR-conditional imputation,
most notably on ETTh1.

A key limitation is the exponential growth of the state space in $N^{DK^*}$,
which can make full kernel estimation challenging in high-dimensional settings.
In practice, KARMA mitigates this via marginal kernels, sampling strategies
restricting to $\mathcal{H}^+_{K^*}$ (bounded linearly by $T - K^*$) and
thresholding on $\hat{\pi}^*(h)$, with reliability certified by the Level~5
coverage map; scaling to very high-dimensional systems remains an open
direction. We believe the Markov approximation paradigm offers a principled and
underexplored direction for XAI in sequential domains, with applications
spanning finance, neuroscience, and genomics.

\bibliographystyle{unsrtnat}
\bibliography{references}  





\newpage
\mbox{}
\newpage
\section*{Supplementary Material}
This supplementary material gathers the proofs and implementation details
deferred from the main text: the soundness of the trimming threshold $\lambda$
and the K.E.R.I.\ reliability index (Section~\ref{app:lambda}); finite-sample
total-variation guarantees, noise floors, and sample budgets for the three
kernel estimators (Section~\ref{sec:step4}); the model-induced causal graph and
faithfulness of KARMA's attributions (Section~\ref{app:dag}); and the full
algorithm, the $T>1$ extension, and the evaluation metric. References prefixed
with ``S'' point to this supplement; unprefixed ones to the main paper.
\setcounter{section}{0}                 
\renewcommand{\thesection}{S\arabic{section}}        
\renewcommand{\thesubsection}{S\arabic{section}.\arabic{subsection}}   
\renewcommand{\thesubsubsection}{S\arabic{section}.\arabic{subsection}.\arabic{subsubsection}}

\setcounter{figure}{0}\renewcommand{\thefigure}{S\arabic{figure}}      
\setcounter{table}{0}\renewcommand{\thetable}{S\arabic{table}}         
\setcounter{equation}{0}\renewcommand{\theequation}{S\arabic{equation}} 
\section{Constructive Extension to $T>1$}
Recall $[N]$, $[D]$, $\psi$, $\mathcal S$ and $\mathcal{H}_K$ as defined in
Section~3.1 of the manuscript. For a forecast horizon of length $T$, let
$\mathcal{S}_{T} = [N]^{D \times T}$ denote the extended (discretised) output
state space, so that a discretised forecast is a matrix
$s = (s^{d}_{i})_{d \in [D],\, i \in [T]} \in \mathcal{S}_{T}$, where $s^{d}_{i}$
is the bin of variable $d$ at horizon step $i$. Given the selected order $K^*$,
the \textbf{multi-step (joint) surrogate transition kernel} of $f$ is
\begin{equation}
    \mathcal T^{f}_{K^*}(s \mid h) \;=\;
    \mathbb P\!\left(\psi(f(\tilde{h})) = s \;\middle|\;
    \mathrm{suffix}_{K^*}(\psi(\tilde{h})) = h\right),
    \label{eq:multi_joint_kernel}
\end{equation}
for $s \in \mathcal{S}_{T}$ and $h \in \mathcal{H}_{K^*}$, where $\tilde h$ is a
$W$-length input window whose discretised $K^*$-suffix equals $h$.

\paragraph{Horizon paths.}
For $T>1$ an explanation tracks a single target variable at each future step.
We encode this choice by a \emph{relevant mask}
$M \in \{0,1\}^{D \times T}$ with exactly one active entry per horizon step,
\begin{equation}
    M[d][i] =
    \begin{cases}
        1 & d = d_i,\\
        0 & \text{otherwise,}
    \end{cases}
    \qquad i \in [T],
\end{equation}
where $d_i \in [D]$ is the variable selected at step $i$ (the $d_i$ need not be
distinct across steps). Writing $\langle s, M\rangle = (s^{d_i}_{i})_{i \in [T]}
\in [N]^{T}$ for the entries of $s$ picked out by $M$, the
\textbf{multi-step marginal surrogate transition kernel} along the path $M$ is
\begin{equation}
    \mathcal T^{f,M}_{K^*}(s^{T} \mid h) \;=\;
    \mathbb P\!\left(\langle \psi(f(\tilde{h})),\, M\rangle = s^{T}
    \;\middle|\;
    \mathrm{suffix}_{K^*}(\psi(\tilde{h})) = h\right),
    \label{eq:multi_marginal_kernel}
\end{equation}
where $s^{T} \in [N]^{T}$ is the length-$T$ horizon trajectory selected by $M$.
Because $M$ activates at most one variable per step, the path contains no
contemporaneous (same-step) couplings, consistent with the strictly-past
construction of Section~3.2; the single-target case $T=1$ recovers the marginal
kernel $\mathcal T^{f,d}_{K^*}$ of Eq.~(4).

\subsection{The Five-Level Explanation Hierarchy}
\label{subsec:multi-5-level}
Let $K^* >0$, $h \in \mathcal H_{K^*}$, $d \in [D]$ and $f$ a trained model. Further assume $\hat{\mathcal T}_{K^*}^{f,M}$ the approximation of $\mathcal T_{K^*}^{f,M}$ such that 
\begin{equation}
\label{eq:rho_floor_h_multi}
    \mathbb{E}\!\left[
    \|\hat{\mathcal T}^{f,M}_K(\cdot \mid h)-\;
               \mathcal T^{f,M}_K(\cdot \mid h)\|_{TV}
  \right] < \frac{\lambda}{4}, \qquad \lambda > 0
\end{equation}
 We describe two cases, user-specified prediction time steps

\begin{definition}[Lag-Resolved Influence and Variable Importance]
We are given the estimated horizon kernel $\hat{\mathcal T}^{f}_{K^*}(s\mid h)$
over $\mathcal S_{T}=[N]^{D\times T}$ (Eq.~\ref{eq:multi_joint_kernel}) and wish
to quantify the effect on $q$ chosen forecast steps
$\tau=\{t_1,\dots,t_q\}\subseteq[T]$.
Fix a target dimension $d$. Selecting its
length-$T$ horizon with the dimension mask $M_d\in\{0,1\}^{D\times T}$
($M_d[d'][i]=\mathbf 1\{d'=d\}$) gives
\[
  \hat{\mathcal T}^{f,d}_{K^*}(s^{d}_{1:T}\mid h)
  = \mathbb P\!\left(\langle\psi(f(\tilde h)),M_d\rangle=s^{d}_{1:T}
    \mid \mathrm{suffix}_{K^*}(\psi(\tilde h))=h\right),
  \quad s^{d}_{1:T}\in[N]^{T},
\]
already computed in equation (19) and marginalizing out the $T-q$ steps outside $\tau$ keeps the requested steps:
\[
  \hat{\mathcal T}^{f,d,\tau}_{K^*}(s^{d}_{\tau}\mid h)
  = \sum_{\mathbf x\in[N]^{[T]\setminus\tau}}
    \hat{\mathcal T}^{f,d}_{K^*}\!\bigl(s^{d}_{\tau},\,
      s^{d}_{[T]\setminus\tau}=\mathbf x\mid h\bigr),
  \quad s^{d}_{\tau}\in[N]^{q}.
\]
The edge strength of source $d'$ at lag $k$ on target $d$ over the steps $\tau$
is
\begin{equation}
\label{eq:multi-rho}
  \rho(X^{d'}_{t_1-k}\to X^{d}_{\tau})
  = \sum_{h\in\mathcal{H}_{K^*}}\hat{\pi}^*(h)\cdot\tfrac12\,\Delta_{TV},
  \qquad
  \Delta_{TV}=\!\!\sum_{s^{d}_{\tau}\in[N]^{q}}\!\!
  \left|\hat{\mathcal T}^{f,d,\tau}_{K^*}(s^{d}_{\tau}\mid h)
    -\frac1N\sum_{x=0}^{N-1}
    \hat{\mathcal T}^{f,d,\tau}_{K^*}(s^{d}_{\tau}\mid h^{d'\leftarrow x}_k)\right|,
\end{equation}
where $h^{d'\leftarrow x}_k$ replaces only the $d'$-th component at lag $k$.

The \emph{lag-resolved influence} of source
$d'$ at lag $k$ sums the edge strengths over all targets $d$,
\begin{figure}
    \centering
    \begin{tikzpicture}[
  x=1cm,y=1cm,
  cell/.style   ={draw,line width=0.5pt,minimum size=5mm,inner sep=0pt,fill=white},
  fcell/.style  ={draw,line width=0.5pt,minimum size=5mm,inner sep=0pt,fill=violet!16},
  selc/.style   ={draw,line width=0.5pt,minimum size=5mm,inner sep=0pt,fill=violet!38},
  srccell/.style={draw=brown!80!black,line width=0.6pt,minimum size=5mm,inner sep=0pt,fill=brown!55},
  lbl/.style    ={font=\small},
  brace/.style  ={decorate,decoration={brace,amplitude=4pt},line width=0.6pt},
]
\def\XB{6.6}
\def\rowdots{-1.24}
 
\fill[brown!22,rounded corners=2pt] (0.28,0.34) rectangle (0.96,-2.20);
 
\foreach \x in {0,1.86,2.48}{
  \foreach \y in {0,-0.62,-1.86}{ \node[cell] at (\x,\y) {}; }}
\foreach \y in {0,-0.62,-1.86}{ \node[srccell] at (0.62,\y) {}; }  
\foreach \y in {0,-0.62,-1.86}{ \node at (1.24,\y) {$\cdots$}; }
\foreach \x in {0,0.62,1.86,2.48}{ \node at (\x,\rowdots) {$\vdots$}; }
\node at (1.24,\rowdots) {$\ddots$};
\draw[rounded corners=4pt,line width=0.8pt] (-0.36,0.36) rectangle (2.84,-2.22);
\draw[brace] (-0.31,0.46) -- (2.79,0.46) node[midway,yshift=12pt]{\small$W$};
\draw[brace,decoration={mirror}] (-0.62,0.31) -- (-0.62,-2.17) node[midway,xshift=-10pt]{\small$D$};
\node[lbl] at (1.24,1.95) {$W\times D$ input window};
\node[lbl] at (0.62,-2.62) {lag $k$};
 
\draw[-{Stealth[length=4mm,width=4mm]},line width=2.2pt] (3.05,-0.93) -- (\XB-0.55,-0.93);
\node at ($(3.05,-0.93)!0.5!(\XB-0.55,-0.93)+(0,0.42)$) {$f$};
\node[font=\scriptsize\itshape] at ($(3.05,-0.93)!0.5!(\XB-0.55,-0.93)+(0,-0.42)$) {predictor};
 
\foreach \x in {0,0.62,1.86}{
  \foreach \y in {0,-0.62,-1.86}{ \node[cell] at (\XB+\x,\y) {}; }}
\fill[violet!20,rounded corners=2pt] (\XB-0.05,-0.35) rectangle (\XB+1.55,-0.89); 
\foreach \x in {0,0.62}{ \node[selc] at (\XB+\x,-0.62) {}; }   
\foreach \y in {0,-0.62,-1.86}{ \node at (\XB+1.24,\y) {$\cdots$}; }
\foreach \x in {0,0.62,1.86}{ \node at (\XB+\x,\rowdots) {$\vdots$}; }
\node at (\XB+1.24,\rowdots) {$\ddots$};
\draw[rounded corners=4pt,line width=0.8pt] (\XB-0.36,0.36) rectangle (\XB+2.22,-2.22);
\node[lbl] at (\XB+0,0.62)    {$t_1$};
\node[lbl] at (\XB+0.62,0.62) {$t_2$};
\node[lbl] at (\XB+1.24,0.62) {$\cdots$};
\node[lbl] at (\XB+1.86,0.62) {$T$};
\draw[brace] (\XB-0.31,1.02) -- (\XB+1.6,1.02)
      node[midway,yshift=11pt]{\small$\tau=\{t_1,\dots,t_q\}$};
\node[lbl,anchor=west] at (\XB+2.30,-0.62) {$d$};   
\node[lbl] at (\XB+0.93,1.98) {$T\times D$ forecast horizon};
 
\draw[-{Stealth[length=3mm,width=3mm]},line width=1pt]
  (0.62,-2.95)
    .. controls (1.5,-4.6) and (\XB+3.4,-4.6) ..
  (\XB+3.4,-1.6)
    .. controls (\XB+3.4,-0.62) and (\XB+2.7,-0.62) ..
  (\XB+2.6,-0.62);
\node[lbl] at (3.95,-3.55) {$\rho\!\left(X^{d'}_{t_1-k}\to X^{d}_{\tau}\right)$};
 
\end{tikzpicture}
\caption{\textbf{KARMA explanation for a $T>1$ forecast.} A black-box predictor
$f$ maps a $W\times D$ input window (window length $W$, $D$ variables) to a
$T\times D$ forecast horizon. The explanation targets a user-chosen subset of
horizon steps $\tau=\{t_1,\dots,t_q\}\subseteq[T]$; the remaining $T-q$ steps
are marginalised out of the estimated horizon kernel. The shaded input column is
the lag-$k$ source slice, and the highlighted row-$d$ band over $\tau$ is the
target. The curved arrow denotes the interventional edge
$\rho\!\left(X^{d'}_{t_1-k}\to X^{d}_{\tau}\right)$ measured by KARMA: the
$\hat{\pi}^*$-weighted total-variation shift in the targeted forecast
$\hat{\mathcal T}^{f,d,\tau}_{K^*}(\cdot\mid h)$ when the $d'$-th input component
at lag $k$ is intervened on. Lags $k>K^*$ are certified negligible by the
stopping rule.}
\label{fig:karma-horizon}
\end{figure}
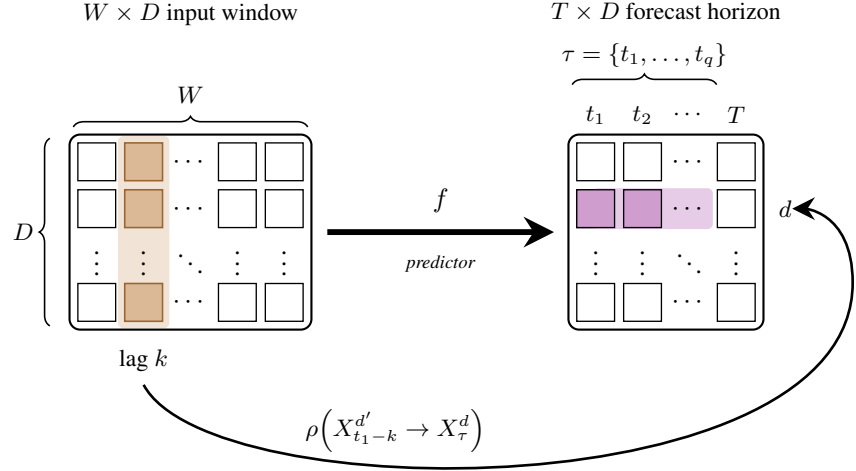
\begin{equation}
\label{eq:multi-lag-resolved-influence}
  \phi_k^{d'} = \sum_{d=1}^{D}
  \rho(X_{t-k}^{d'}\to X_{\tau}^{d})\cdot
  \mathbf 1\{\rho(X_{t-k}^{d'}\to X_{\tau}^{d})>\lambda\},
\end{equation}
and the variable importance sums over lags, $\Phi^{d'}=\sum_{k=1}^{K^*}\phi_k^{d'}$,
normalised as in the $T=1$ case. The remaining levels follow as in the $T=1$ case.
\end{definition}

\begin{remark}[Semantics of cross-variable, cross-time edges]
\label{rem:cross-variable-cross-time}
By construction, the surrogate conditions only apply to strictly past states, so
KARMA admits no \emph{contemporaneous} cross-variable links: two distinct
variables at the same time slice, $X^{d'}_{t-k}\to X^{d}_{t-k}$ ($d'\neq d$),
are excluded by design (cf.\ Definition~\ref{app:dag}). Cross-variable
influence is therefore expressible only when separated by a delay $k\ge 1$,
and every such edge $X^{d'}_{t-k}\to X^d_t$ makes precisely one claim:
variable $d'$ exerts a \emph{direct} influence on variable $d$ that
materialises after $k$ steps and that no other variable--lag pair in the
Markov blanket can account for. Its strength $\rho(e)$ quantifies this
residual effect, while the lag $k$ carries an equally concrete meaning, it is
the \emph{propagation delay} with which $d'$ acts on $d$. Such delays are
physical rather than statistical: the transatlantic contagion lag in finance,
the axonal conduction delay in neuroscience, the generation-plus-travel time
in epidemiology, or the transcription delay from a regulator to its target in
genomics. KARMA reads $k$ as a learned estimate of this delay, not an artifact
of estimation. The stopping rule sharpens the reading: edges at lags $k>K^*$
are certified negligible, whereas those at $k\le K^*$ are precisely the
mechanistic pathways the model has learned to exploit. Contemporaneous
couplings, where one variable drives another within the same step, fall
outside this surrogate; we leave such instantaneous extensions to future work.
\end{remark}
\section{Causal Sufficiency of Thresholding Constant $\lambda$}
\label{app:lambda}
 
Let $K^* \ge 1$ and $e = (X_{t-k}^{d'} \to X_t^d)$ denote a causal edge and $\hat{\rho}(\rho)$, the quantity in Eq (6) w.r.t. $\hat{\mathcal T}^{f, d}_{K^*}({\mathcal T}^{f, d}_{K^*})$ respectively. In the following, the argument proceeds analogously for the joint kernel ${\mathcal T}^{f}_{K^*}$ as well.
\begin{proposition}
Let ${\mathcal T}^{f,d}_{K^*}(\cdot \mid h)$ be the true model $K-$lag surrogate kernel and $\hat{\mathcal T}^{f,  d}_{K^*}(\cdot \mid h)$ it's approximation via any strategy. The constant $\lambda > 0 $ is chosen such that 

\[
 \mathbb{E}\!\left[
    \|\hat{\mathcal T}^{f,d}_{K^*}(\cdot \mid h)-\;
               \mathcal T^{f,d}_{K^*}(\cdot \mid h)\|_{TV}
  \right] < \frac{\lambda}{4}.
\]
for any $h \in \mathcal{H}^+_{K^*}$. Further assume $\max_{h \in \mathcal{H}^+_{K^*}} \rho_{\mathrm{floor}}(h) < \lambda/4$, where $\rho_{\mathrm{floor}}(h)$ is an approximate upper bound for $\mathbb{E}\!\left[
    \|\hat{\mathcal T}^{f,d}_K(\cdot \mid h)-\;
               \mathcal T^{f,d}_K(\cdot \mid h)\|_{TV}
  \right]$. Then the trimming decision at threshold $\lambda$ in Level~1 is correct with probability at least $1/2$ for each edge, and the probability of a incorrect trimming decision is bounded by:
\begin{equation}
   \mathbb P\left(|\hat{\rho}(e) - \rho(e)| \geq \lambda/2\right) \leq \frac{4\max_{h}\rho_{\mathrm{floor}}(h)}{\lambda} < 1.
    \label{eq:markov_bound}
\end{equation}
\end{proposition}

\begin{proof}

By definition of $\hat{\rho}$ and $\rho$ from Eq.~(6), using $\big||a| - |b|\big| \leq |a - b|$ and then the triangle inequality to the inner difference:
\begin{align}
|\hat{\rho}(e) - \rho(e)| &\leq \sum_{h} \hat{\pi}^*(h) \cdot \frac{1}{2}\sum_{s^d} \left[ \left|\hat{\mathcal T}^{f,d}_{K^*}(s^d|h) - \mathcal T^{f,d}_{K^*}(s^d|h)\right| \right. \notag \\
&\quad \left. + \frac{1}{N}\sum_x \left|\hat{\mathcal T}^{f,d}_{K^*}(s^d|h^{d'\leftarrow x}_k) - \mathcal T^{f,d}_{K^*}(s^d|h^{d'\leftarrow x}_k)\right| \right].
\end{align}

Recognizing the inner sums as total variation distances:
\begin{align}
|\hat{\rho}(e) - \rho(e)| &\leq\sum_{h} \hat{\pi}^*(h) \Bigg[ 
\left[\|\hat{\mathcal T}^{f,d}_{K^*}(\cdot|h) - \mathcal T^{f,d}_{K^*}(\cdot|h)\|_{TV}\right] + \frac{1}{N}\sum_{x} \Big[\|\hat{\mathcal T}^{f,d}_{K^*}(\cdot|h^{d'\leftarrow x}_k) - \mathcal T^{f,d}_{K^*}(\cdot|h^{d'\leftarrow x}_k)\|_{TV}\Big] \Bigg].
\end{align}

Taking expectations on both sides and applying linearity, noting that $\hat{\pi}^*(h)$ is estimated independently of the kernel:
\begin{align}
\mathbb{E}\left[|\hat{\rho}(e) - \rho(e)|\right] 
&\leq \sum_{h} \hat{\pi}^*(h) \Bigg[ 
\mathbb{E}\left[\|\hat{\mathcal T}^{f,d}_{K^*}(\cdot|h) - \mathcal T^{f,d}_{K^*}(\cdot|h)\|_{TV}\right] \notag \\
&\quad + \frac{1}{N}\sum_{x} \mathbb{E}\Big[\|\hat{\mathcal T}^{f,d}_{K^*}(\cdot|h^{d'\leftarrow x}_k)- \mathcal T^{f,d}_{K^*}(\cdot|h^{d'\leftarrow x}_k)\|_{TV}\Big] \Bigg].
\end{align}

Since counterfactual histories $h^{d'\leftarrow x}_k \in \mathcal{H}^+_{K^*}$ for all $x \in [N]$. Since $\sum_h \hat{\pi}^*(h) = 1$:
\begin{align}
\mathbb{E}\left[|\hat{\rho}(e) - \rho(e)|\right] &\leq \sum_{h} \hat{\pi}^*(h) \cdot 2\cdot\max_{h'}\rho_{\mathrm{floor}}(h') \notag \\
&= 2\cdot\max_{h \in \mathcal{H}^+_{K^*}} \rho_{\mathrm{floor}}(h).
\end{align}

Since $|\hat{\rho}(e) - \rho(e)| \geq 0$, Markov's inequality applies directly:
\begin{equation}
    \mathbb P\left(|\hat{\rho}(e) - \rho(e)| \geq t\right) \leq \frac{\mathbb{E}\left[|\hat{\rho}(e) - \rho(e)|\right]}{t} \leq \frac{2\cdot\max_{h}\rho_{\mathrm{floor}}(h)}{t}.
\end{equation}

Setting $t = \lambda/2$:
\begin{equation}
     \mathbb P\left(|\hat{\rho}(e) - \rho(e)| \geq \lambda/2\right) \leq \frac{4\cdot\max_{h}\rho_{\mathrm{floor}}(h)}{\lambda}.
    \label{eq:markov_applied}
\end{equation}

We now verify that the trimming decision at threshold $\lambda$ is correct with high probability under the condition $\max_h \rho_{\mathrm{floor}}(h) < \lambda/4$. Consider two cases:

\begin{itemize}
    \item \textbf{False negative}: A true edge with $\rho(e) \geq \lambda$ is missed because $\hat{\rho}(e) < \lambda$. This requires $|\hat{\rho}(e) - \rho(e)| \geq \lambda/2$, since if the true value is at least $\lambda$ away from $\lambda/2$, an error of at most $\lambda/2$ cannot push the estimate below $\lambda$. By Eq.~\eqref{eq:markov_applied}, this occurs with probability at most $\frac{4\max_h \rho_{\mathrm{floor}}(h)}{\lambda} < 1$.

    \item \textbf{False positive}: A spurious non-edge with $\rho(e) = 0$ is included because $\hat{\rho}(e) \geq \lambda$. This requires $|\hat{\rho}(e) - \rho(e)| \geq \lambda$, a stricter condition. By Markov's inequality with $t = \lambda$:
    \begin{equation}
        \mathbb P\left(|\hat{\rho}(e) - \rho(e)| \geq \lambda\right) \leq \frac{2\cdot\max_{h}\rho_{\mathrm{floor}}(h)}{\lambda} < \frac{1}{2},
    \end{equation}
    under the condition $\max_h \rho_{\mathrm{floor}}(h) < \lambda/4$.
\end{itemize}

Therefore, under $\max_h \rho_{\mathrm{floor}}(h) < \lambda/4$, both false negatives and false positives occur with probability strictly less than $1$, and the probability of any incorrect trimming decision across all edges is bounded by Eq.~\eqref{eq:markov_bound}. Substituting the condition $\max_h \rho_{\mathrm{floor}}(h) < \lambda/4$ into Eq.~\eqref{eq:markov_applied} gives the safety guarantee:
\begin{equation}
   \mathbb P\left(|\hat{\rho}(e) - \rho(e)| \geq \lambda/2\right) < 1,
\end{equation}
confirming $\lambda$ as a valid trimming threshold provided the noise floor condition $\max_h \rho_{\mathrm{floor}}(h) < \lambda/4$ is satisfied. 
\end{proof}
\subsection{KARMA Explanation Reliability Index}
\label{app:keri}

We derive the probabilistic foundation of K.E.R.I from the
Markov bound established in Proposition~1. Recall that for any
causal edge $e = (d', k, d)$:
\begin{equation}
    \mathbb{P}\left(|\hat{\rho}(e) - \rho(e)| \geq \frac{\lambda}{2}\right)
    \leq \frac{4\cdot\max_{h}\rho_{\mathrm{floor}}(h)}{\lambda},
    \label{eq:markov_keri}
\end{equation}
which is derived by applying Markov's inequality to the corrected
bound:
\begin{equation}
    \mathbb{E}\left[|\hat{\rho}(e) - \rho(e)|\right]
    \leq 2\cdot\max_{h \in H^{+}_{K^*}} \rho_{\mathrm{floor}}(h).
\end{equation}
For the bound in Eq.~\eqref{eq:markov_keri} to be non-vacuous,
i.e.\ strictly less than 1, we require:
\begin{equation}
    \max_{h \in H^{+}_{K^*}} \rho_{\mathrm{floor}}(h)
    < \frac{\lambda}{4}.
    \label{eq:keri_condition}
\end{equation}

Since $\rho_{\mathrm{floor}}(h)$ is itself an upper bound on the
true per-history estimation error, the condition in
Eq.~\eqref{eq:keri_condition} may be overly conservative: the
true noise floor is likely smaller than $\rho_{\mathrm{floor}}(h)$
by some factor $\kappa \geq 1$. To account for this inflation, we
relax the non-vacuity condition by replacing $\lambda/4$ with
$\kappa \cdot \lambda/4$:
\begin{equation}
    \max_{h \in H^{+}_{K^*}} \rho_{\mathrm{floor}}(h)
    < \kappa \cdot \frac{\lambda}{4},
    \label{eq:keri_condition_kappa}
\end{equation}
and redefine the per-history miscoverage probability bound
accordingly:
\begin{equation}
    \delta_{\kappa}(h) := \min\left(
        \frac{4\,\rho_{\mathrm{floor}}(h)}{\kappa\cdot\lambda},\; 1
    \right),
\end{equation}
so that $\delta_{\kappa}(h) \in [0,1]$ is the worst-case
probability that the trimming decision at history $h$ is
incorrect under the relaxed threshold. The conservatism constant
$\kappa > 1$ reflects the practitioner's belief that
$\rho_{\mathrm{floor}}(h)$ overestimates the true noise floor;
$\kappa = 1$ recovers the unrelaxed bound, and $\kappa = 2$ is
the recommended default. K.E.R.I aggregates the complementary
per-history reliability $1 - \delta_{\kappa}(h)$ over the
stationary distribution:
\begin{align}
    \mathrm{K.E.R.I}
    &= \sum_{\hat{\pi}^*(h) > 0} \hat{\pi}^*(h)
       \cdot \bigl(1 - \delta_{\kappa}(h)\bigr) \notag\\
    &= \sum_{\hat{\pi}^*(h) > 0} \hat{\pi}^*(h)
       \cdot \left(1 -
           \frac{4\,\rho_{\mathrm{floor}}(h)}{\kappa\cdot\lambda}
       \right)_{\!+},
\end{align}
which is the $\hat{\pi}^*$-weighted expected reliability of the
$\lambda$-trimming decisions across the full history space.
A value of $\mathrm{K.E.R.I}$ close to $1$ indicates that the
non-vacuity condition~\eqref{eq:keri_condition_kappa} is
comfortably satisfied at most histories; a low value suggests
increasing $M$ (more Monte Carlo draws) or $\lambda$ until the
condition is met. Note that $\mathrm{K.E.R.I}$ is a lower bound
on reliability for any fixed $\kappa$: increasing $\kappa$
relaxes the threshold and raises the reported score, so the
practitioner should choose $\kappa$ conservatively.

\paragraph{Interpretation.}
K.E.R.I $= 1$ if and only if $\rho_{\mathrm{floor}}(h) = 0$ for 
all $h \in \mathcal{H}^+_{K^*}$, an ideal case requiring perfect 
kernel estimation. In practice, K.E.R.I $\in [0, 1)$, and values 
close to 1 indicate that trimming decisions are well-supported 
across the stationary distribution. Since $\rho_{\mathrm{floor}}(h)$ 
is itself an upper bound, K.E.R.I is conservative: the true 
expected reliability is at least as large as the reported value. 
We therefore recommend interpreting K.E.R.I as a \emph{lower 
bound on reliability}. A low K.E.R.I does not invalidate the 
certified-zero attributions for lags $k > K^*$, which rest 
solely on $\hat{\Delta}^{\mathrm{pred}} < \varepsilon$ and are 
unaffected by kernel estimation quality.

\begin{remark}
Setting $\lambda < 0.15$ is sufficient to ensure a non-vacuous edge
decisions under Eq.~\eqref{eq:keri_condition_kappa} with $\kappa = 10$
for all estimation strategies in Section~5.4.
\end{remark}

\begin{remark}[Empirical robustness beyond the sufficient condition]
The condition $\rho_{\mathrm{floor}}(h) <  \cdot \lambda/4$ in
Proposition~A.1 is \emph{sufficient} but not \emph{necessary} for correct
edge recovery. In our synthetic experiments, KARMA recovered the true causal
structure even for histories where this condition was violated. This is
consistent with two sources of conservatism in the bound: (i)
$\rho_{\mathrm{floor}}(h)$ is itself an upper bound on the true per-history
estimation error $\|\hat{\mathcal T}^{f,d}_{K^*}(\cdot|h) -
\mathcal T^{f,d}_{K^*}(\cdot|h)\|_{\mathrm{TV}}$, which may be substantially smaller
in practice; and (ii) correct thresholding of $\rho(e)$ requires only that
the $\hat{\pi}^*$-weighted average error across all $h \in \mathcal H^{+}_{K^*}$ is
small relative to the gap between true edge strengths and $\lambda$, a
strictly weaker condition than the per-history guarantee in
Proposition~A.1. Accordingly, \textrm{K.E.R.I} even with the $\kappa$-relaxation, should be interpreted as a
conservative lower bound on explanation reliability: a low value warrants
caution but does not certify that edge decisions are incorrect. The uncertainty estimates of level 5 still offer the best reliability estimates.
\end{remark}

\section{Step 4: (Marginal) Transition Kernel Estimation}
\label{sec:step4}
We expand in detail the strategies for kernel estimation introduced in the manuscript.

\subsection{Strategy A: Empirical Counting (Dirichlet Estimator)}
\label{sec:strategy-a}

The empirical estimator accumulates model outputs directly from held-out
data. For each query $i$ on window $\tilde{h}_i$, record the
$K^*$-length discretised suffix $h_i$ and the discretised output
$s^{(i)} = \psi(f(\tilde{h}_i))$. Define the counts
\begin{equation*}
  n^{f}(h, s)
  = \bigl|\bigl\{i \in [N] : h_i = h,\;
    \psi(f(\tilde{h}_i)) = s\bigr\}\bigr|,
\end{equation*}

$$ n^f(h) = \sum_{s} n^{f}(h, s).$$

 The Dirichlet posterior mean with Jeffreys prior $\alpha\ge 0$ gives:
\begin{equation}
  \hat{\mathcal T}^{f, (\mathrm A)}_{K^*}(s \mid h)
  = \frac{n^{f}(h, s) + \alpha}{n^f(h) + N^D\alpha}.
  \label{eq:dirichlet-estimator-joint}
\end{equation}

 A sufficient minimum number of observations needed to estimate $\mathcal T^{f, (\mathrm A)}_K(\cdot|h)$ up to maximum expected total variation error $\frac{\lambda}{4} >0$ for any $h \in \mathcal{H}^+_K$ is $T^{(\mathrm A)}_{\min}= W + \frac{N^D}{(\frac{\lambda}{4})^2} \cdot N^{D K^*}$ (Theorem~\ref{thm:Total Variation Error of Estimation of Joint Kernel}) and grows exponentially in $D K^* +D$ the number of histories blows up faster than the per-history count requirement, infeasible at standard
financial or clinical series lengths.

To remain tractable, \textsc{KARMA} estimates $D$ \textbf{factored
marginal kernels} instead of the full joint.
For target variable $d \in [D]$, history $h \in \mathcal{H}^+_K$,
and next bin $s^{d} \in [N]$:
\begin{equation}
  \mathcal T^{f,d, (\mathrm A)}_K(s^{d} \mid h)
  = \mathbb{P}\!\left(
      \psi^d\!\bigl(f(\tilde{h})\bigr) = s^{d}
      \;\Big|\;
      \mathrm{suffix}_K(\psi(\tilde{h})) = h
    \right),
  \label{eq:factored-kernel}
\end{equation}
the marginal of $\mathcal T^{f, (\mathrm A)}_K$ over the $d$-th output coordinate.
Each model query $(h, s)$ contributes to all $D$ marginals
simultaneously by reading off $s^{d}$ for each $d$,
yielding a reduction in the data
requirement per history.


 Define the marginal counts:
\begin{equation*}
  n^{f,d}(h, s^{d})
  = \bigl|\bigl\{i \in [N] : h_i = h,\;
    \psi^d(f(\tilde{h}_i)) = s^{d}\bigr\}\bigr|,
\end{equation*}
$$ n^f(h) = \sum_{s^{d}} n^{f,d}(h, s^{d}).$$
 The Dirichlet posterior mean with Jeffreys prior $\alpha \ge 0$ becomes:
\begin{equation}
  \hat{\mathcal T}^{f,d, (\mathrm A)}_{K^*}(s^{ d} \mid h)
  \approx \frac{n^{f,d}(h, s^{d}) + \alpha}{n^f(h) + N\alpha}.
  \label{eq:dirichlet-estimator}
\end{equation}
Each query contributes to all $D$ marginals simultaneously by reading off
$s^{(i)}_d$ for each $d$, giving a factor $N^{D-1}$ count improvement
over the joint kernel. This strategy with the marginal kernel reduces $T^{(\mathrm A)}_{\min} $ to $W + \frac{N}{(\frac{\lambda}{4})^2} \cdot N^{D K^*}$ (Theorem~\ref{thm:noise-floor-A2}), better but still not an easily attainable budget for large $N, D$ and/or $K^*$.

\begin{theorem}[Total Variation Error of Estimation of Joint Kernel]
\label{thm:Total Variation Error of Estimation of Joint Kernel}

 Let $\hat{\mathcal T}^{f, (\mathrm A)}_K(\cdot \mid h)$ denote the Dirichlet posterior mean
estimator (Eq.~\ref{eq:dirichlet-estimator-joint}) with Jeffreys prior
$\alpha = 1/2$, based on $n^f(h)$ observations. The following inequality holds true
\begin{equation}
\label{eq:kernel-estimation-error-joint}
  \mathbb{E}\!\left[
    \|\hat{\mathcal T}^{f, (\mathrm A)}_K(\cdot \mid h)-\;
               \mathcal T^{f}_K(\cdot \mid h)\|_{TV}
  \right]
  \;\leq\;
  \sqrt{\frac{N^D}{2\,n^f(h)}},
\end{equation} 

We define the \textit{noise floor} as
\[
\rho^{(\mathrm A)}_{\text{floor}}(h): = \sqrt{\frac{N^D}{2\, n^f(h)}}
\]

Setting this below $(\frac{\lambda}{4}) \geq 0$:
\[
\sqrt{\frac{N^D}{2\, n^f(h)}} < (\frac{\lambda}{4})
\;\;\Longrightarrow\;\;
n^f(h) > \frac{N^D}{(\frac{\lambda}{4})^2}
\]
and the minimum amount of observed history needed,
\[
T^{(\mathrm A)}_{\min} = W + \frac{N^D}{(\frac{\lambda}{4})^2} \cdot N^{D K^*}
\]
\end{theorem}
Let $n = n^f(h)$, $p_s = \mathcal T^{f,d}_K(\cdot \mid h)$
denote the true kernel, and
$\hat{p}_s = (n_s + \alpha)/(n + N^D\alpha)$
the Dirichlet posterior mean with Jeffreys prior $\alpha \ge 0$,
where $n_s = n^{f,d}(h, s)$ and $(n_s)_{s \in [N]}$. 
\begin{lemma}[Pinsker's Inequality]
     Let $P$ and $Q$ be two probability distributions on the same measurable space. 
Then Pinsker's inequality states:

\[
\|P-Q\|_{TV} \leq \sqrt{\frac{1}{2} \mathrm{KL}(P \| Q)} \leq \sqrt{\frac{1}{2} \chi^2(P \| Q)}
\]

where
\begin{align*}
\mathrm{KL}(P \| Q) &:= \sum_{x} P(x) \log \frac{P(x)}{Q(x)}\\
\chi^2(P \| Q) &:=   \sum_{x} \frac{(P(x) - Q(x))^2}{Q(x)}
\end{align*}

are the Kullback--Leibler divergence and chi-squared divergence, respectively.
 \end{lemma}

\begin{proof}[Proof of Theorem~\ref{thm:Total Variation Error of Estimation of Joint Kernel}]
 Decompose the estimation error:
\[
  \hat{p}_s - p_s
  = \frac{(n_s - np_s) + \alpha(1 - N^Dp_s)}{n + N^D\alpha}.
\]
Since $\mathbb{E}[n_s] = np_s$, the cross term vanishes and:
\begin{align*}
  \mathbb{E}\bigl[(\hat{p}_s - p_s)^2\bigr]
  &= \frac{\mathrm{Var}(n_s) + \alpha^2(1 - N^Dp_s)^2}{(n + N^D\alpha)^2}
  \\ &= \frac{np_s(1 - p_s) + \alpha^2(1 - N^Dp_s)^2}{(n + N^D\alpha)^2},
\end{align*}
using $\mathrm{Var}(n_s) = np_s(1 - p_s)$ for Binomial$(n, p_s)$.

Summing over $s$ and dividing by $p_s$:
\begin{align}
  \mathbb{E}\bigl[\chi^2(\hat{p}\|p)\bigr]
  &= \frac{1}{(n + N^D\alpha)^2} \times\sum_{s=0}^{N-1}
     \frac{np_s(1-p_s) + \alpha^2(1-N^Dp_s)^2}{p_s} \notag\\
  &= \frac{n}{(n + N^D\alpha)^2}
       \sum_s (1 - p_s) + \frac{\alpha^2}{(n + N^D\alpha)^2}
       \sum_s \frac{(1-N^Dp_s)^2}{p_s} \\
  &= \frac{1}{(n + N^D\alpha)^2}
     \left[n(N^D-1) 
     + \alpha^2 \sum_s
     \frac{(1-N^Dp_s)^2}{p_s}\right].
    \label{eq:chi2-sum}
\end{align}
For the bias term, expand
$(1 - N^Dp_s)^2/p_s = 1/p_s - 2N^D + N^{2D} p_s$
and use $\sum_s p_s = 1$ and $p_s \leq 1$:
\[
  \sum_s \frac{(1-N^{D}p_s)^2}{p_s}
  \leq \sum_s \frac{1}{p_s} - 2N^{D+1} + N^{2D}
  \le \sum_s \frac{1}{p_s} +N^{2D}.
\]

Retaining only the leading term in~\eqref{eq:chi2-sum}:
\begin{align}
  \mathbb{E}\bigl[\chi^2(\hat{p}\|p)\bigr]
  &\;\leq\;
  \frac{n(N^{D}-1) + \alpha^2 C_p}{(n + N^{D}\alpha)^2}\notag
  \\&\;\leq\;
  \frac{N^{D}}{n}
  \cdot\frac{1 + \alpha^2 C_p / (nN^{D})}{(1 + N^{D}\alpha/n)^2},
  \label{eq:chi2-bound}
\end{align}
where $C_p = \sum_s (1+N^{D}p_s)^2/p_s \geq 0$.
For $n$ sufficiently large (specifically $n \geq N^D\alpha$
and $n \geq \alpha^2 C_p / N^{D}$ or setting $\alpha \to 0$):
\begin{equation}
  \mathbb{E}\bigl[\KL(\hat{p}\|p)\bigr]
  \;\leq\;
  \mathbb{E}\bigl[\chi^2(\hat{p}\|p)\bigr]
  \;\leq\;
  \frac{N^{D}}{n^f(h)},
  \label{eq:kl-bound}
\end{equation}
 The KL divergence from the posterior mean to the true kernel satisfies:
\begin{equation}
  \mathrm{KL}\!\left(\hat{\mathcal T}^{f, (\mathrm A)}_K(\cdot \mid h)
               \,\Big\|\,
               \mathcal T^{f}_K(\cdot \mid h)\right)
  \;\leq\;
  \frac{N^D}{n^f(h)},
  \label{eq:kl-dirichlet-joint}
\end{equation}
for $n^f(h)$ sufficiently large.
Applying Pinsker's Inequality to~\eqref{eq:kl-dirichlet-joint} yields:
\begin{equation}
  \mathbb{E}\!\left[
    \|\hat{\mathcal T}^{f, (\mathrm A)}_K(\cdot \mid h)-\;
              \mathcal T^{f}_K(\cdot \mid h)\|_{TV}
  \right]
  \;\leq\;
  \sqrt{\frac{N^D}{2\,n^f(h)}}
\end{equation}

If we define the \textit{noise floor} as
\[
\rho^{(\mathrm A)}_{\text{floor}}(h): = \sqrt{\frac{N^D}{2\, n^f(h)}}
\]

Setting this below $(\frac{\lambda}{4}) \geq 0$:
\[
\sqrt{\frac{N^D}{2\, n^f(h)}} < (\frac{\lambda}{4})
\;\;\Longrightarrow\;\;
n^f(h) > \frac{N^D}{(\frac{\lambda}{4})^2}
\]

 Meaning every history $h \in \mathcal{H}^+_{K^*}$ needs at least $\frac{N^D}{(\frac{\lambda}{4})^2}$ model queries.  
The total number of histories in the worst case is:
\[
\left|\mathcal{H}_{K^*}\right| = N^{D K^*}
\]

Let $T$ be the number of observations used in the training set (or whatever held-out set is used). Since each window contributes to exactly one history's count, the total number of windows needed is:
\[
T - W \geq \frac{N^D}{(\frac{\lambda}{4})^2} \cdot N^{D K^*}
\]

giving:
\[
T^{(\mathrm A)}_{\min} \approx W + \frac{N^D}{(\frac{\lambda}{4})^2} \cdot N^{D K^*}, \quad n^f(h) > \frac{N^D}{(\frac{\lambda}{4})^2}
\]

the sufficient minimum amount of history needed.
\end{proof}
\begin{theorem}
\label{thm:noise-floor-A2}
    Let $\hat{\mathcal T}^{f,d}_K(\cdot \mid h)$ denote the Dirichlet posterior mean
estimator (Eq.~\ref{eq:dirichlet-estimator}) with Jeffreys prior
$\alpha = 1/2$, based on $n^f(h)$ observations.

\begin{equation}
  \mathbb{E}\!\left[
    \|\hat{\mathcal T}^{f,d}_K(\cdot \mid h)-\;
              \mathcal T^{f,d}_K(\cdot \mid h)\|_{TV}
  \right]
  \;\leq\;
  \sqrt{\frac{N}{2\,n^f(h)}},
  \label{eq:noisefloor-pinsker}
\end{equation}
which is the noise floor of Strategy~A (Eq.~\ref{eq:noisefloor-a}).
By Pinsker's inequality applied to the Dirichlet posterior:
\begin{equation}
  \rho^{(\mathrm A)}_{\mathrm{floor}}(h):
  = \sqrt{\frac{N}{2\,n^f(h)}},
  \label{eq:noisefloor-a}
\end{equation}
which diverges as $n^f(h) \to 0$.
And the sufficient minimum series length if we require $ \rho^{(\mathrm A)}_{\mathrm{floor}}(h) < (\frac{\lambda}{4})$, $(\frac{\lambda}{2})\ge 0$:
\begin{equation}
  T^{(\mathrm A)}_{\min}
  \approx W + \frac{N}{(\frac{\lambda}{4})^2} \cdot N^{DK^*},\quad n^f(h) > \frac{N}{(\frac{\lambda}{4})^2}
  \label{eq:tmin-a}
\end{equation}
growing \emph{exponentially} in $DK^*$.
\end{theorem}
\begin{proof}[Proof of Theorem~\ref{thm:noise-floor-A2}]
    Follows from the fact that:
\begin{equation}
  \mathrm{KL}\!\left(\hat{\mathcal T}^{f,d, (\mathrm A)}_K(\cdot \mid h)
               \,\Big\|\,
              \mathcal T^{f,d}_K(\cdot \mid h)\right)
  \;\leq\;
  \frac{N}{n^f(h)},
  \label{eq:kl-dirichlet}
\end{equation}
for $n^f(h)$ sufficiently large. Everything else follows as in the proof of Theorem~\ref{thm:Total Variation Error of Estimation of Joint Kernel}
\end{proof}
\subsection{Strategy B: $b^*$-Prefix Monte Carlo (default in practice)}
\label{sec:strategy-b}
 Even in the marginal case, the Dirichlet estimator still proves inconvenient as we require observation lengths of order $N^{(DK^* + 1)}$. So instead of observed history instances, we look at observed history distribution. We define the following quantity,
 
\begin{definition}[Within-bin conditional density]
\label{def:within-bin}
Let $p_{\mathrm{stat}} : \mathbb{R}^{D \times K^*} \to \mathbb{R}_{\geq 0}$
denote the stationary joint density of the $K^*$-length suffix
$(X_{t-K^*+1}, \ldots, X_t)$ under the data-generating process.
The \textbf{within-bin conditional density} at history $h$ is:
\begin{equation}
  p(x \mid \psi(x) = h)
  \;=\;
  \frac{p_{\mathrm{stat}}(x)\;\mathbf{1}[\psi(x) = h]}
       {\hat{\pi}^*(h)},
  \quad x \in \mathbb{R}^{D \times K^*},
  \label{eq:within-bin-density}
\end{equation}
where $\hat{\pi}^*(h) = \mathbb{P}(\psi(X_{t-K^*+1:t}) = h) > 0$
is the stationary probability of history $h$.
This density characterises the distribution of continuous states
conditional on falling in the bins specified by $h$.
Sampling from $p(\cdot \mid \psi(\cdot) = h)$ gives the authentic
within-bin continuous variation that both estimation strategies
must approximate.
\end{definition}
\begin{definition}[Suffix pool]
\label{def:pool}
For $h \in \mathcal{H}^+_{K^*}$, the \textbf{suffix pool} is:
\begin{equation}
  \mathrm{pool}(h)
  = \bigl\{
      \tilde{x} \in \mathbf{X}_{\mathrm{train}} :
      \psi(\tilde{x}_{W-K^*+1:W}) = h
    \bigr\},
  \label{eq:pool}
\end{equation}
the set of all $W$-length training windows whose discretised $K^*$-suffix
equals $h$.
Sampling uniformly from $\mathrm{pool}(h)$ provides a non-parametric
Monte Carlo approximation to the within-bin conditional density
(Definition~\ref{def:within-bin}).
\end{definition}

\noindent Strategy B proceeds as follows:

\noindent For each $h \in \mathcal{H}^+_{K^*}$ and $m = 1,\ldots,M$:
\begin{enumerate}[leftmargin=*, itemsep=1pt, topsep=2pt]
  \item Sample $\tilde{x}^{(m)}\sim\mathrm{Uniform}(\mathrm{pool}(h))$
        (Definition~\ref{def:pool}).
        This approximates the within-bin conditional density
        (Definition~\ref{def:within-bin}) non-parametrically from training data.
  \item Replace the prefix with $b^*$:
        $\tilde{h}^{(m)} = b^* \oplus \tilde{x}^{(m)}_{W-K^*+1:W}$,
        where $b^* \in \mathbb{R}^{(W-K^*) \times D}$ is the
        model-certified baseline from Step~3.
  \item Query the model: $s^{(m)}_d = \psi^d(f(\tilde{h}^{(m)}))$.
\end{enumerate}
The kernel estimate is:
\begin{equation}
  \hat{\mathcal T}^{f,d, (\mathrm B)}_{K^*}(s^{d} \mid h)
  = \frac{\sum_{m=1}^M \mathbf{1}[s^{(m)}_d = s^{d}] + \alpha}
         {M + N\alpha},
  \quad \alpha = \tfrac{1}{2}.
  \label{eq:bstar-kernel}
\end{equation}
\begin{theorem}[Strategy~B noise floor decomposition]
\label{prop:noisefloor-b}
Let $\hat{\mathcal T}^{f,d,(\mathrm{B})}_{K^*}(\cdot \mid h)$ be the
$b^*$-prefix Monte Carlo estimator (Eq.~\ref{eq:bstar-kernel}) with
$M$ draws and pool $\mathrm{pool}(h)$ of size $P_h = |\mathrm{pool}(h)|$.
Then:
\begin{equation}
  \mathbb{E}\!\left[
    \|
      \hat{\mathcal T}^{f,d, (\mathrm B)}_{K^*}(\cdot \mid h)-\;
      \mathcal T^{f,d}_{K^*}(\cdot \mid h)
    \|_{TV}
  \right]
  \leq \rho^{(\mathrm B)}_{\mathrm{floor}}(h)
\label{eq:noisefloor-b-bound}
\end{equation}
where, 

\[
\rho^{(\mathrm B)}_{\mathrm{floor}}(h)=\underbrace{\sqrt{\frac{\pi}{2M}}}_{\text{(i) MC variance}}
  +\;
  \underbrace{\hat{\Delta}^{\mathrm{pred}}}_{\text{(ii) prefix error}}
  +\;
  \underbrace{\sqrt{\frac{\pi}{2P_h}}}_{\text{(iii) within-bin var.}}
\]
And the minimum series length if we require $ \rho^{(\mathrm B)}_{\mathrm{floor}}(h) < (\frac{\lambda}{4})$, $\frac{\lambda}{4}> 0$:
\[
T^{(\mathrm B)}_{\min}
\approx W + \left\lceil \frac{2\pi}{
\lambda^2}\right\rceil \cdot \left|\mathcal{H}^+_{K^*}\right|, \quad \text{and,} \:
P_h >  \left\lceil \frac{2\pi}{\lambda^2}\right\rceil\]
\end{theorem}
\begin{proof}[Proof of Theorem~\ref{prop:noisefloor-b}]
Assume the Jeffrey prior $\alpha = 0$, the case when $\alpha =1/2$ follows via upper bound approximation. Introduce two intermediate distributions.
Let $\mathcal T^{f,d, b^*}_{K^*}(\cdot \mid h)$ denote the kernel estimated using
the true within-bin density but with prefix replaced by $b^*$:
  $\mathcal T^{f,d, b^*}_{K^*}(s^{d} \mid h)$ given by 
\begin{equation}
  \mathbb{E}_{x \sim p(\cdot \mid \psi(x)=h)}\!\left[
      \mathbf{1}\!\left[\psi^d\!\bigl(f(b^* \oplus x_{W-K^*+1:W})\bigr)
      = s^{d}\right]
    \right],
  \label{eq:kernel-bstar-pop}
\end{equation}
where the expectation is over the true within-bin conditional density
(Definition~\ref{def:within-bin}).
Let $\mathcal T^{f,d,\mathrm{pool}}_{K^*}(\cdot \mid h)$ denote the same
quantity with the pool empirical distribution replacing the true
within-bin density:

$\mathcal T^{f,d,\mathrm{pool}}_{K^*}(s^{d} \mid h)$ given by
\begin{equation}
  \frac{1}{P_h}\sum_{\tilde{x} \in \mathrm{pool}(h)}
    \mathbf{1}\!\left[\psi^d\!\bigl(f(b^* \oplus \tilde{x}_{W-K^*+1:W})\bigr)
    = s^{ d}\right].
  \label{eq:kernel-pool-pop}
\end{equation}
The estimator $\hat{\mathcal T}^{f,d, (\mathrm B)}_{K^*}$ is a Monte Carlo
approximation to~\eqref{eq:kernel-pool-pop} with $M$ draws.
By the triangle inequality:
\begin{align*}
  \|
    \hat{\mathcal T}^{f,d, (\mathrm B)}_{K^*}-\;
    \mathcal T^{f,d}_{K^*}
  \|_{TV}
  &\leq
  \underbrace{
    \|
      \hat{\mathcal T}^{f,d, (\mathrm B)}_{K^*}-\;
      \mathcal T^{f,d,\mathrm{pool}}_{K^*}
    \|_{TV}
  }_{\text{(i)}}
  \\ &+\;
  \underbrace{
    \|
      \mathcal T^{f,d,\mathrm{pool}}_{K^*}-\;
      \mathcal T^{f,d, b^*}_{K^*}
    \|_{TV}
  }_{\text{(ii)}}
 \\ & +\;
  \underbrace{
    \|
      \mathcal T^{f,d, b^*}_{K^*}-\;
      \mathcal T^{f,d}_{K^*}
    \|_{TV}
  }_{\text{(iii)}}.
  \label{eq:triangle}
\end{align*}

\begin{itemize}
    \item For the term (i): $\hat{\mathcal T}^{f,d, (\mathrm B)}_{K^*}(\cdot \mid h)$ estimates
$T^{f,d,\mathrm{pool}}_{K^*}(\cdot \mid h)$ by averaging $M$ i.i.d.\ Bernoulli
indicators in $[0,1]$ for each $s^{d} \in [N]$.
By Hoeffding's inequality, each marginal probability is estimated within
$t$ with probability at least $1 - 2e^{-2Mt^2}$.
Taking expectation over the $M$ draws:
\begin{equation}
  \mathbb{E}\!\left[
   \|
      \hat{\mathcal T}^{f,d, (\mathrm B)}_{K^*}-\;
      \mathcal T^{f,d,\mathrm{pool}}_{K^*}
    \|_{TV}
  \right]
  \;\leq\; \sqrt{\frac{\pi}{2M}}.
  \label{eq:term-i}
\end{equation}

\item For the term (iii): $\mathcal T^{f,d, b^*}_{K^*}$ uses the true within-bin density but the $b^*$ prefix,
while $\mathcal T^{f,d}_{K^*}$ uses the true within-bin density and the true prefix.
The Total Variation distance between them is bounded by the average prediction
discrepancy when the prefix is replaced by $b^*$:

$$  \|
    \mathcal T^{f,d, b^*}_{K^*}(\cdot \mid h)-\;
   \mathcal T^{f,d}_{K^*}(\cdot \mid h)
  \|_{TV}$$
  $$\;\leq\;
  \mathbb{E}_{x \sim p(\cdot|\psi(x)=h)}\!\left[
    \ell\!\left(f(\tilde{h}_{\mathrm{true}}),\;
               f(b^* \oplus x_{W-K^*+1:W})\right) \right]$$
$$ 
  \;\leq\; \hat{\Delta}^{\mathrm{pred}},$$
 
where the last inequality follows from the surrogate validity
certificate (Eq.~\ref{eq:equiv}), which bounds the average prediction
discrepancy under prefix substitution by $\hat{\Delta}^{\mathrm{pred}} < \varepsilon$.

\item For the term (ii): $\mathcal T^{f,d,\mathrm{pool}}_{K^*}$ uses the empirical pool distribution
as a surrogate for the true within-bin density.
The pool is a random sample of size $P_h$ from the within-bin
conditional distribution, so by the Dvoretzky--Kiefer--Wolfowitz (DKW)
inequality \cite{dvoretzky1956} applied to the empirical distribution:
\begin{equation}
  \mathbb{E}\!\left[
    \|
     \mathcal T^{f,d,\mathrm{pool}}_{K^*}(\cdot \mid h)-\;
      \mathcal T^{f,d, b^*}_{K^*}(\cdot \mid h)
    \|_{TV}
  \right]
  \;\leq\; \sqrt{\frac{\pi}{2P_h}},
  \label{eq:term-ii}
\end{equation}
where $P_h = |\mathrm{pool}(h)|$.
Intuitively, the pool is an empirical approximation to the within-bin
density; the TV error of this approximation decreases as
$P_h^{-1/2}$, the standard Monte Carlo rate for distribution estimation.
\end{itemize}

 The noise floor $\rho_{\text{floor}}(h)$ has three terms:
\[
\rho_{\text{floor}}(h)
= \sqrt{\frac{\pi}{2M}} + \hat{\Delta}_{\text{pred}} + \sqrt{\frac{\pi}{2P_h}}
\]

The first two terms are controlled by design choices $M$ and $\varepsilon$ independently of $T$.  
The binding constraint from $T$ comes entirely from the pool size term $\sqrt{\frac{\pi}{2P_h}}$.

Subtracting the other two terms from the budget $\frac{\lambda}{4} > 0$:
\[
\sqrt{\frac{\pi}{2P_h}} < 
\frac{\lambda}{4}  - \sqrt{\frac{\pi}{2M}} - \hat{\Delta}_{\text{pred}} < \frac{\lambda}{4} 
\]

requires:
\[
P_h >  \left\lceil \frac{2\pi}{\lambda^2}\right\rceil =: n_{\text{pool}}
\]

where $\lceil a \rceil$ is the nearest biggest integer to $a \in \mathbb{R}$. Now $P_h = |\text{pool}(h)|$ is the number of training windows whose discretised $K^*$-suffix equals $h$.  
Each of the $T - W$ available windows lands in exactly one pool. For every $h \in \mathcal{H}^+_{K^*}$ to accumulate $n_{\text{pool}}$ windows:
\[
T - W \geq n_{\text{pool}} \cdot \left|\mathcal{H}^+_{K^*}\right|
\]

giving:
\[
T^{(\mathrm B)}_{\min}
\approx W + n_{\text{pool}} \cdot \left|\mathcal{H}^+_{K^*}\right|
\]

The key difference from Strategy A is that:
\[
\left|\mathcal{H}^+_{K^*}\right| \leq T - K^*,
\]

that is, the observed support is at most the number of distinct suffixes seen in the data, which grows linearly in $T$ rather than exponentially in $D K^*$. This is what gives Strategy B its linear scaling.
\end{proof}

\begin{remark}[Practical reliability below the sufficient sample budget]
The sample thresholds $T^{(A)}_{\min}$ and $T^{(B)}_{\min}$ are
\emph{sufficient} conditions for kernel estimation error to fall below
$\lambda/4$; reliable transition kernel estimates are routinely obtained
in practice with fewer observations. It is therefore critical to report
Level~5 of the explanation hierarchy alongside any KARMA output, as the
epistemic variance map directly certifies the reliability of the
estimated kernel and hence of all Level~1--4 explanations under the
actual data budget.
\end{remark}
\subsection{Strategy C: Regression Tree on the Transition Tensor}
\label{sec:strategy-c}

Strategies~A and~B both estimate $\mathcal T^{f,d}_{K^*}(\cdot \mid h)$
independently for each $h \in \mathcal{H}^+_{K^*}$, treating every
history as a separate estimation problem.
This nonparametric approach is interpretable and certified but
data-hungry: the pool requirement scales as
$n_{\mathrm{pool}}^{(\mathrm B)} \cdot |\mathcal{H}^+_{K^*}|$, where $n_{\mathrm{pool}}^{(\mathrm B)}$ is the minimum pool size such that $\rho^{(\mathrm B)}_{\mathrm{floor}}(h) \le \beta$ for all $h$.
Strategy~C replaces independent per-history estimation with a
\emph{regression tree} that partitions $\mathcal{H}^+_{K^*}$ into
leaves and pools counts within each leaf, sharing statistical strength
across structurally similar histories while remaining fully
interpretable.

\paragraph{Covered and sparse sets.}
Partition $\mathcal{H}^+_{K^*}$ according to whether each history
has accumulated sufficient model counts:
\[
  \mathcal{H}^{\mathrm{cov}}
  = \bigl\{h \in \mathcal{H}^+_{K^*} : P_h \geq n^{(\mathrm B)}_{\mathrm{pool}}\bigr\},
  \qquad
  \mathcal{H}^{\mathrm{sparse}}
  = \mathcal{H}^+_{K^*} \setminus \mathcal{H}^{\mathrm{cov}}.
\]
Histories in $\mathcal{H}^{\mathrm{cov}}$ have enough data to
estimate their individual transition distributions reliably at the
noise floor; histories in $\mathcal{H}^{\mathrm{sparse}}$ do not.
The regression tree is grown on $\mathcal{H}^{\mathrm{cov}}$ and
extrapolates to $\mathcal{H}^{\mathrm{sparse}}$ via leaf assignment.

\paragraph{What the tree is doing.}

The tree operates in three conceptually distinct steps.

\textbf{Step 1 (direct estimation).}
For every $h \in \mathcal{H}^{\mathrm{cov}}$, compute the
individual kernel estimate $\hat{\mathcal T}^{f,d, (\mathrm C)}_{K^*}(\cdot \mid h)$
directly from B.
The covered set exists precisely because these estimates are
reliable: $P_h \geq n_{\mathrm{pool}}$ ensures the noise floor
at each $h \in \mathcal{H}^{\mathrm{cov}}$ is already below the
coverage threshold.

\textbf{Step 2 (structured pooling).}
Let $h \in \mathcal{H}^{\mathrm{cov}} \subseteq [N]^{D \times K}$ and target dimension $d$,
the input is the discrete history vector $h$ be represented as
$(h^{d'}_k)_{d' \in [D],\, k \in [K^*]} \in [N]^{DK^*}$, $h^{d'}_k$ denoting the bin of variable $d' \le d$ at lag $k$ within $h$. The tree partitions $\mathcal{H}^{\mathrm{cov}}$ into $L$ leaves
by recursively selecting binary splits of the form
$\{h : h^{d'}_k = n\}$ vs.\ $\{h : h^{d'}_k \neq n\}$
 for any random $(d', k, n) \in [D] \times [K^*] \times [N]$ such that the split minimises the within-leaf heterogeneity of the individual kernel
estimates (Eq.~\ref{eq:tree-split}).
Within each leaf $\ell$, model counts are aggregated across all
member histories to form a \textbf{pooled leaf kernel}
$\bar{\mathcal T}^d_\ell$ (Eq.~\ref{eq:tree-leaf-kernel}),
which replaces the individual estimates for every covered history
in that leaf.
This pooling is structured rather than arbitrary: the splits are
chosen so that histories assigned to the same leaf have similar
individual kernels, making the pooled kernel a good representative
for all of them.
The cost of pooling is the within-leaf bias $\delta_\ell$ ,
the maximum Total Variation distance between any individual kernel in the leaf
and the pooled kernel, which the KL
The splitting criterion is directly minimized.

\textbf{Step 3 (extrapolation by structural similarity).}
For $h \in \mathcal{H}^{\mathrm{sparse}}$, there are insufficient
model counts to form a reliable individual estimate.
The tree routes $h$ to a leaf $\ell(h)$ by evaluating the same
binary split conditions top-down from the root: at each internal
node, check whether $h^{d^*}_{k^*} = n^*$, $d^* \le d$ and go left if true,
right otherwise.
The sparse history then inherits the leaf's pooled kernel.
The implicit claim is: $h$ shares the same bin values as the
covered histories in $\ell(h)$ on the variables and lags that the
tree has identified as predictively important, so their transition
distributions should be similar.
This is extrapolation by structural similarity,
not smoothing, not interpolation, but stratification:
histories that look alike on the splits that matter are assigned
the same transition distribution estimate.

\paragraph{Input and response.}
For each $h \in \mathcal{H}^{\mathrm{cov}} \subseteq [N]^{D \times K}$ and target dimension $d$,
the input is the discrete history vector
$x_h = (h^{d'}_k)_{d' \in [D],\, k \in [K^*]} \in [N]^{DK^*}$, $h^{d'}_k$ denoting the bin of variable $d'$ at lag $k$ within $h$ 
and the response is the individually estimated kernel simplex
$y_h = \hat{\mathcal T}^{f,d}_{K^*}(\cdot \mid h)$.

\paragraph{Tree construction.}
The tree recursively partitions $\mathcal{H}^{\mathrm{cov}}$ by
choosing splits of the form
$\{h : h^{d'}_k = n\}$ vs.\ $\{h : h^{d'}_k \neq n\}$
for each variable $d' \in [D]$, lag $k \in [K^*]$, and bin
$n \in [N]$, giving $D \cdot K^* \cdot N$ candidates per node.
At each node the split minimising the $\hat{\pi}^*$-weighted
within-leaf KL divergence is selected:
\begin{equation}
  \mathrm{c}_{\mathrm{split}}
  = \sum_{\ell \in \{\mathrm{left,\,right}\}}
    \sum_{h \in \ell} \hat{\pi}^*(h)\cdot
    \KL\!\left(
      \hat{\mathcal T}^{f,d}_{K^*}(\cdot \mid h)
      \;\Big\|\;
      \bar{\mathcal T}^d_\ell(\cdot)
    \right),
  \label{eq:tree-split}
\end{equation}
where the \textbf{pooled leaf kernel} is:
\begin{equation}
  \bar{\mathcal T}^d_\ell(s^{d})
  = \frac{\displaystyle\sum_{h \in \ell}
          n^{f,d}(h,\, s^{d}) + \alpha}
         {\displaystyle\sum_{h \in \ell} n^f(h) + N\alpha},
  \quad \alpha = \tfrac{1}{2}.
  \label{eq:tree-leaf-kernel}
\end{equation}

Splitting halts when any of the following hold:
the node contains a single history;
maximum depth $d_{\max}$ is reached;
no valid split exists that gives both children pooled count
$\geq n_{\mathrm{pool}}^{(\mathrm C)}$;
or the best cost reduction is below threshold $\eta$.
The pool-size check $\sum_{h \in \ell} P_h \geq n_{\mathrm{pool}}^{(\mathrm C)}$
at every candidate child is the structural link between tree
construction and $T^{(\mathrm{C})}_{\min}$: it guarantees that
every leaf produced by the algorithm meets the coverage requirement
by construction, without any post-hoc pruning.

\paragraph{Extrapolation to sparse histories.}
For $h \in \mathcal{H}^{\mathrm{sparse}}$, route $h$ down the tree
to its leaf $\ell(h)$ and assign:
\begin{equation}
  \hat{\mathcal T}^{f,d, (\mathrm C)}_{K^*}(\cdot \mid h)
  = \bar{\mathcal T}^d_{\ell(h)}(\cdot).
  \label{eq:tree-estimator}
\end{equation}
The leaf $\ell(h)$ groups histories sharing the same bin conditions
on the predictively important splits, providing a non-parametric
but a structured estimate for histories that would otherwise be
inaccessible to Strategies~A and~B.

\begin{algorithm}
\caption{Strategy~C: Tree construction }
\label{alg:tree-construction-d}
\begin{algorithmic}[1]
\Require
  Individual kernel estimates
  $\{\hat{\mathcal T}^{f,d}_{K^*}(\cdot \mid h)\}_{h \in \mathcal{H}^{\mathrm{cov}}}$,
  counts $\{n^{f,d}(h, s^{d}),\, n^f(h)\}_{h \in \mathcal{H}^{\mathrm{cov}}}$,
  stationary weights $\{\hat{\pi}^*(h)\}_{h \in \mathcal{H}^{\mathrm{cov}}}$,
  pool sizes $\{P_h\}_{h \in \mathcal{H}^{\mathrm{cov}}}$,
  parameters $n_{\mathrm{pool}},\, d_{\max},\, \eta,\, \alpha = \tfrac{1}{2}$
\Ensure
  Tree $\mathcal{T}^d$ with leaf kernels
  $\{\bar{T}^d_\ell\}$ and split conditions
  $\{(d'^*_\nu,\, k^*_\nu,\, n^*_\nu)\}_{\nu\,\text{internal}}$

\medskip

\State Initialise root $\nu_0$:\;
  $\mathcal{H}(\nu_0) \leftarrow \mathcal{H}^{\mathrm{cov}}$,\;
  $\mathrm{depth}(\nu_0) \leftarrow 0$
\State $\mathcal{Q} \leftarrow \{\nu_0\}$

\While{$\mathcal{Q} \neq \emptyset$}

  \State Pop node $\nu$ from $\mathcal{Q}$

  \State Compute pooled kernel and node cost:
  \begin{align*}
    \bar{\mathcal T}^d_\nu(s^{ d})
    &\leftarrow
    \frac{\displaystyle\sum_{h \in \mathcal{H}(\nu)}
          n^{f,d}(h,\, s^{ d}) + \alpha}
         {\displaystyle\sum_{h \in \mathcal{H}(\nu)} n^f(h) + N\alpha}
    \qquad \forall\, s^{d} \in [N]
    \\[4pt]
    c_\nu
    &\leftarrow
    \sum_{h \in \mathcal{H}(\nu)}
    \hat{\pi}^*(h)\cdot
    \KL\!\Bigl(
      \hat{\mathcal T}^{f,d}_{K^*}(\cdot \mid h)
      \;\Big\|\;
      \bar{\mathcal T}^d_\nu
    \Bigr)
  \end{align*}

  \If{$|\mathcal{H}(\nu)| = 1$
      \;\textbf{or}\;
      $\mathrm{depth}(\nu) \geq d_{\max}$}
    \State Mark $\nu$ as leaf with kernel $\bar{\mathcal T}^d_\nu$;\;
      \textbf{continue}
      \hfill\Comment{trivial stopping}
  \EndIf

  \State $c^* \leftarrow c_\nu$,\quad
    $\sigma^* \leftarrow \mathrm{null}$

  \For{$(d',\, k,\, n)
    \in [D] \times [K^*] \times [N]$}
    \hfill\Comment{$D \cdot K^* \cdot N$ candidates}

    \State $\mathcal{H}_L \leftarrow
      \{h \in \mathcal{H}(\nu) : h^{d'}_k = n\}$,\quad
      $\mathcal{H}_R \leftarrow
      \mathcal{H}(\nu) \setminus \mathcal{H}_L$

    \If{$\mathcal{H}_L = \emptyset$
        \;\textbf{or}\;
        $\mathcal{H}_R = \emptyset$}
      \State \textbf{continue}
    \EndIf

    \If{$\sum_{h \in \mathcal{H}_L} P_h < n_{\mathrm{pool}}$
        \;\textbf{or}\;
        $\sum_{h \in \mathcal{H}_R} P_h < n_{\mathrm{pool}}$}
      \State \textbf{continue}
      \hfill\Comment{child pool violates coverage requirement}
    \EndIf

    \State Compute child pooled kernels
      $\bar{\mathcal T}^d_L$,\, $\bar{\mathcal T}^d_R$
      via Eq.~\eqref{eq:tree-leaf-kernel}

    \State $c_{\mathrm{split}} \leftarrow
      \displaystyle\sum_{\ell \in \{L,\, R\}}
      \sum_{h \in \mathcal{H}_\ell}
      \hat{\pi}^*(h) \cdot
      \KL\!\bigl(
        \hat{\mathcal T}^{f,d}_{K^*}(\cdot \mid h)
        \;\|\;
        \bar{\mathcal T}^d_\ell
      \bigr)$
      \hfill\Comment{Eq.~\eqref{eq:tree-split}, single variable $d$}

    \If{$c_{\mathrm{split}} < c^*$}
      \State $c^* \leftarrow c_{\mathrm{split}}$,\quad
        $\sigma^* \leftarrow
        (d',\, k,\, n,\,
        \mathcal{H}_L,\, \mathcal{H}_R)$
    \EndIf

  \EndFor

  \If{$\sigma^* = \mathrm{null}$
      \;\textbf{or}\;
      $c_\nu - c^* < \eta$}
    \State Mark $\nu$ as leaf with kernel $\bar{\mathcal T}^d_\nu$
    \hfill\Comment{no valid split or gain $< \eta$}
  \Else
    \State $(d'^*,\, k^*,\, n^*,\,
      \mathcal{H}_L,\, \mathcal{H}_R)
      \leftarrow \sigma^*$
    \State $\nu.\mathrm{split}
      \leftarrow (d'^*,\, k^*,\, n^*)$
    \State Create children $\nu_L$,\, $\nu_R$:\;
      $\mathcal{H}(\nu_L) \leftarrow \mathcal{H}_L$,\;
      $\mathcal{H}(\nu_R) \leftarrow \mathcal{H}_R$,\;
      $\mathrm{depth} \leftarrow \mathrm{depth}(\nu) + 1$
    \State $\mathcal{Q} \leftarrow
      \mathcal{Q} \cup \{\nu_L,\, \nu_R\}$
  \EndIf

\EndWhile

\State \Return tree $\mathcal{T}^d$

\end{algorithmic}
\end{algorithm}

\noindent
Algorithm~\ref{alg:tree-construction-d} is run independently for
each target variable $d \in [D]$, producing $D$ trees
$\{\mathcal{T}^d\}_{d=1}^D$ with potentially different split
structures. The split candidates $(d', k, n)$ range over all
source variables $d' \in [D]$, lags $k \in [K^*]$, and bins
$n \in [N]$, so the tree for target variable $d$ may select
splits on any source variable $d'$, including $d' = d$
(self-influence) and $d' \neq d$ (cross-variable influence).
The split variables selected across the $D$ trees directly
encode the Level~1 and Level~2 explanation: a source variable
$d'$ at lag $k$ that appears as a high-level split node in
$\mathcal{T}^d$ is a primary driver of the transition
distribution of variable $d$.

\begin{figure}
    \centering
    \begin{tikzpicture}[
  x=1cm,y=1cm,
  intnode/.style={draw,rounded corners=3pt,fill=blue!5,inner sep=4pt,
                  font=\small,align=center,line width=0.7pt,minimum width=20mm},
  root/.style   ={draw,rounded corners=3pt,fill=black!7,inner sep=4pt,
                  font=\small,align=center,line width=0.8pt},
  edge/.style   ={line width=0.7pt},
  elab/.style   ={font=\scriptsize\itshape,fill=white,inner sep=1pt},
  note/.style   ={font=\footnotesize,align=left},
]
 
\newcommand{\leaf}[6]{%
  \draw[rounded corners=2pt,line width=0.6pt,fill=white]
        (#1-0.62,#2-0.48) rectangle (#1+0.62,#2+0.48);
  \draw[line width=0.4pt] (#1-0.48,#2-0.32) -- (#1+0.48,#2-0.32);
  \fill[blue!55] (#1-0.42,#2-0.32) rectangle (#1-0.20,#2-0.32+#4);
  \fill[blue!55] (#1-0.11,#2-0.32) rectangle (#1+0.11,#2-0.32+#5);
  \fill[blue!55] (#1+0.20,#2-0.32) rectangle (#1+0.42,#2-0.32+#6);
  \node[font=\small] at (#1,#2-0.82) {#3};
}
 
\coordinate (R)  at (0,0);
\coordinate (A)  at (-3.4,-2.4);
\coordinate (B)  at (3.4,-2.4);
\coordinate (L1) at (-5.0,-4.9);
\coordinate (L2) at (-1.8,-4.9);
\coordinate (L3) at (1.8,-4.9);
\coordinate (L4) at (5.0,-4.9);
 
\draw[edge] (R) -- (A) node[elab,pos=0.55]{yes};
\draw[edge] (R) -- (B) node[elab,pos=0.55]{no};
\draw[edge] (A) -- (L1) node[elab,pos=0.55]{yes};
\draw[edge] (A) -- (L2) node[elab,pos=0.55]{no};
\draw[edge] (B) -- (L3) node[elab,pos=0.55]{yes};
\draw[edge] (B) -- (L4) node[elab,pos=0.55]{no};
 
\node[root]    at (R) {$\mathcal{H}^{\mathrm{cov}}$\\[1pt]\scriptsize covered histories};
\node[intnode] at (A) {$h^{d_2}_{k_2}=n_2$\,?};
\node[intnode] at (B) {$h^{d_3}_{k_3}=n_3$\,?};
\node[font=\small] at (0,0.95) {split: $h^{d_1}_{k_1}=n_1$\,?};
 
\leaf{-5.0}{-4.9}{$\bar{\mathcal T}^d_{\ell_1}$}{0.50}{0.24}{0.16}
\leaf{-1.8}{-4.9}{$\bar{\mathcal T}^d_{\ell_2}$}{0.18}{0.52}{0.22}
\leaf{1.8}{-4.9}{$\bar{\mathcal T}^d_{\ell_3}$}{0.30}{0.20}{0.42}
\leaf{5.0}{-4.9}{$\bar{\mathcal T}^d_{\ell_4}$}{0.22}{0.30}{0.42}
 
\draw[red!75!black,dashed,line width=1pt,-{Stealth[length=2.4mm]}]
      (-1.3,1.25) to[out=-30,in=120] (R);
\node[red!75!black,font=\scriptsize,align=center] at (-2.5,1.25)
      {sparse $h\in\mathcal{H}^{\mathrm{sparse}}$};
\draw[red!75!black,dashed,line width=1.1pt,-{Stealth[length=2.6mm]}]
      ($(R)+(0.35,-0.45)$) -- ($(B)+(-0.2,0.5)$);
\draw[red!75!black,dashed,line width=1.1pt,-{Stealth[length=2.6mm]}]
      ($(B)+(0.25,-0.5)$) -- ($(L4)+(-0.15,0.55)$);
\node[red!75!black,font=\scriptsize,align=center] at (6.7,-4.9)
      {inherits\\$\bar{\mathcal T}^d_{\ell_4}$};

\end{tikzpicture}
\caption{\textbf{Strategy~C: regression tree on the transition tensor.} The tree
partitions the covered histories $\mathcal{H}^{\mathrm{cov}}$ through binary bin
tests $h^{d'}_{k}=n$, with candidates $(d',k,n)\in[D]\times[K^*]\times[N]$ chosen
to minimise the $\hat{\pi}^*$-weighted within-leaf KL divergence
(Eq.~\ref{eq:tree-split}). Each leaf stores a pooled kernel $\bar{\mathcal
T}^d_{\ell}$ (Eq.~\ref{eq:tree-leaf-kernel}) aggregated from the model counts of
its member histories, shown here as a distribution over the $N$ next-state bins.
Growth halts when a node holds a single history, reaches depth $d_{\max}$, yields
gain below $\eta$, or would give a child pooled count
$\sum_{h\in\ell}P_h<n_{\mathrm{pool}}$. A sparse history
$h\in\mathcal{H}^{\mathrm{sparse}}$ (dashed) is routed top-down through the same
tests to a leaf and inherits its pooled kernel, extending reliable estimates to
histories never sufficiently covered by Strategies~A and~B. The sample budget
then scales with the number of leaves $L\ll|\mathcal{H}^{+}_{K^*}|$ rather than
with $|\mathcal{H}^{+}_{K^*}|$.}
\label{fig:karma-stratC}
\end{figure}
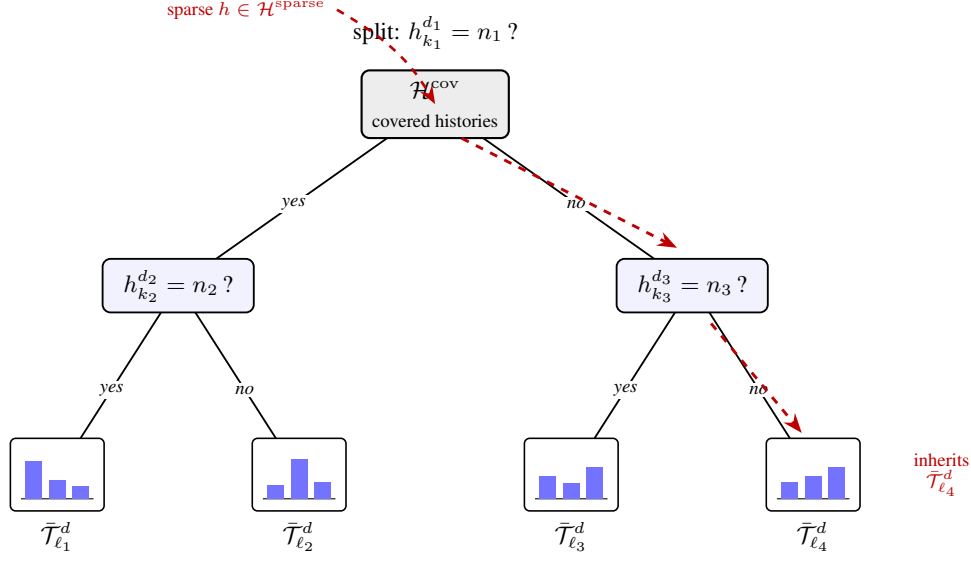
\begin{theorem}[Strategy C noise floor]
Let $\hat{\mathcal T}^{f,d,(C)}_{K^*}(\cdot \mid h)$ be the Strategy C estimator
for $h \in H^+_{K^*}$ routed to leaf $\ell = \ell(h)$ with pooled
count $P_\ell = \sum_{h' \in \ell} P_{h'}$. Define the
\emph{population within-leaf diameter}
\begin{equation}
    \delta_\ell^* \;=\; \max_{h',h''\in\ell}
        \bigl\|\mathcal T^{f,d}_{K^*}(\cdot\mid h') 
             - \mathcal T^{f,d}_{K^*}(\cdot\mid h'')\bigr\|_{TV},
    \label{eq:delta_pop}
\end{equation}
and the population pooled kernel
\begin{equation}
    \mathcal T^{f,d,\mathrm{pool}}_\ell(s^d)
    \;=\; \frac{1}{P_\ell}\sum_{h''\in\ell} P_{h''}\,
           \mathcal T^{f,d}_{K^*}(s^d \mid h'').
    \label{eq:pool_pop}
\end{equation}
Then, for any $h\in\ell$,
\begin{align}
    \mathbb{E}\!\left[
        \bigl\|\hat{\mathcal T}^{f,d,(C)}_{K^*}(\cdot\mid h)
              - \mathcal T^{f,d}_{K^*}(\cdot\mid h)\bigr\|_{TV}
    \right] \notag
    \\\quad\quad\;\leq\;
    \underbrace{\sqrt{\frac{N}{2P_\ell}}}_{\text{(i) pooled estimation}}
    \;+\;
    \underbrace{\delta_\ell^*}_{\text{(ii) within-leaf bias}},
    \label{eq:C_floor_corrected}
\end{align}
so the noise floor is
$\rho^{(C)}_{\mathrm{floor}}(h) = \sqrt{N/(2P_\ell)} + \delta_\ell^*$.
\end{theorem}

\begin{proof}
Fix $h\in H^+_{K^*}$ routed to leaf $\ell$. By definition~\eqref{eq:pool_pop}
and the assignment rule $\hat{T}^{f,d,(C)}_{K^*}(\cdot\mid h)
= \bar{T}^d_\ell$ (Eq.~55 in the main paper), apply the
triangle inequality with $T^{f,d,\mathrm{pool}}_\ell$ as intermediate:

\begin{align}
&\bigl\|\hat{\mathcal T}^{f,d,(C)}_{K^*}(\cdot\mid h)
       - \mathcal T^{f,d}_{K^*}(\cdot\mid h)\bigr\|_{TV} \notag\\
&\quad\leq\;
  \underbrace{
    \bigl\|\bar{\mathcal T}^d_\ell
          - \mathcal T^{f,d,\mathrm{pool}}_\ell\bigr\|_{TV}
  }_{\text{Term (i)}}
  \;+\;
  \underbrace{
    \bigl\|\mathcal T^{f,d,\mathrm{pool}}_\ell
          - \mathcal T^{f,d}_{K^*}(\cdot\mid h)\bigr\|_{TV}
  }_{\text{Term (ii)}}.
  \label{eq:tri}
\end{align}

\medskip
\noindent\textbf{Bounding Term (i).}\;
$\bar{\mathcal T}^d_\ell$ is the Dirichlet posterior mean
(Eq.~54 in the main paper) formed by pooling all $P_\ell$ model
queries within leaf $\ell$. Each query yields a Multinomial
observation in $[N]$. By the same Pinsker--Dirichlet argument
used in Theorem A.4 (Strategy A, marginal kernel), applied
now to the pooled estimator with $P_\ell$ total counts and $N$
bins:
\begin{equation}
    \mathbb{E}\!\left[
        \bigl\|\bar{\mathcal T}^d_\ell
              - \mathcal T^{f,d,\mathrm{pool}}_\ell\bigr\|_{TV}
    \right]
    \;\leq\; \sqrt{\frac{N}{2P_\ell}}.
    \label{eq:term1}
\end{equation}
This is a finite-sample statement requiring no limiting argument.

\medskip
\noindent\textbf{Bounding Term (ii).}\;
By convexity of the total variation distance,
\begin{align}
    (ii)
    &=\Bigl\|
        \frac{1}{P_\ell}\sum_{h''\in\ell}P_{h''}\,
        \mathcal T^{f,d}_{K^*}(\cdot\mid h'') -\mathcal T^{f,d}_{K^*}(\cdot\mid h)
      \Bigr\|_{TV} \notag\\
    &\leq\;
      \frac{1}{P_\ell}\sum_{h''\in\ell}P_{h''}\,
      \biggl[\|\mathcal T^{f,d}_{K^*}(\cdot\mid h'')- \mathcal T^{f,d}_{K^*}(\cdot\mid h)\|_{TV} \biggr]\notag\\
    &\leq\;
      \max_{h''\in\ell}
      \bigl\|\mathcal T^{f,d}_{K^*}(\cdot\mid h'')
            - \mathcal T^{f,d}_{K^*}(\cdot\mid h)\bigr\|_{TV} \notag\\
    &\leq\; \delta_\ell^*,
    \label{eq:term2}
\end{align}
where the last inequality holds because $h \in \ell$ and the
maximum in~\eqref{eq:delta_pop} is taken over all pairs in
$\ell$. Crucially, every step here involves only the
\emph{true} population kernels $\mathcal T^{f,d}_{K^*}$; no estimated
quantity appears, so no finite-sample correction is needed for
this term.

\medskip
\noindent\textbf{Combining.}\;
Taking expectations in~\eqref{eq:tri} and substituting
\eqref{eq:term1}--\eqref{eq:term2}:
\begin{equation}
    \mathbb{E}\!\left[
        \bigl\|\hat{\mathcal T}^{f,d,(C)}_{K^*}(\cdot\mid h)
              - \mathcal T^{f,d}_{K^*}(\cdot\mid h)\bigr\|_{TV}
    \right]
    \;\leq\;
    \sqrt{\frac{N}{2P_\ell}} + \delta_\ell^*,
\end{equation}
establishing~\eqref{eq:C_floor_corrected}. 

\begin{remark}[Why the original proof was incorrect]
The original proof defined $\delta_\ell$ using
\emph{estimated} kernels $\hat{T}^{f,d}_{K^*}(\cdot\mid h')$
and then invoked a limiting argument
($P_{h''}\to\infty$) to transfer this to the population
quantity needed in Term (ii). This conflates estimated and
population objects and provides no finite-sample guarantee for
the substitution error. The corrected proof avoids this by
defining $\delta_\ell^*$ directly in population terms
(Eq.~\eqref{eq:delta_pop}), so Term (ii) is bounded purely
at the population level with no estimation error, and Term (i)
carries the sole finite-sample contribution.
\end{remark}

\begin{remark}[Minimum series length]
Setting $\rho^{(C)}_{\mathrm{floor}}(h) \leq \beta$ and
solving for $P_\ell$ gives
\begin{equation}
    P_\ell \;>\; \frac{N}{2(\beta - \delta_\ell^*)^2}
    \;=:\; n_{\mathrm{pool}},
    \qquad \beta > \delta_\ell^*.
\end{equation}
With $L$ leaves and windows approximately uniformly
distributed, $P_\ell \approx (T-W)/L$, so
\begin{equation}
    T^{(C)}_{\min} \;\approx\; W + n_{\mathrm{pool}}\cdot L,
\end{equation}
which grows in $L \ll |H^+_{K^*}|$ rather than in $|H^+_{K^*}|$,
recovering the sample-size advantage of Strategy C over
Strategy B. Note that $\beta > \delta_\ell^*$ is a necessary
condition: if the leaf is too heterogeneous, no amount of data
can drive the noise floor below the within-leaf bias, and the
tree must be grown deeper to reduce $\delta_\ell^*$.
\end{remark}

\end{proof}

\paragraph{Degradation to Strategy~B.}
When $T$ is large enough that
$\mathcal{H}^{\mathrm{sparse}} = \emptyset$, Strategy~C with
one history per leaf recovers the nonparametric kernel exactly.
Strategy~C therefore degrades gracefully: it activates the
tree structure only where coverage fails and reverts to the
nonparametric estimate for well-covered histories.
\section{Model Induced Causal Graph Recovery}
\label{app:dag}

We refer to DAG, a Directed Acyclic Graph. We distinguish carefully between (i) the \textit{data-generating} causal DAG $G^*$ over the true data-generating process, and (ii) the \textit{model-induced} causal DAG $G^f$ encoding what $f$ has learned. This distinction is not merely
philosophical: if $f$ has learned spurious correlations, $G^f$
faithfully reflects them, which is precisely what model-centric XAI should
report. Comparing $G_f$ with $G^*$ (recovered by PCMCI
on raw data) constitutes a model audit: agreement indicates causally correct
learning; divergence flags what the model got wrong.

\begin{definition}[Model-induced causal DAG]
Given surrogate $\Mks$, define the node set 

\[
V = \{X_{t-k}^d: d \in [D], k \in \{0, \dots K^*\}\}
\]
A directed edge
$X^d_{t-k} \to X^{d'}_{t-k'}$ (with $k > k'$, forward in time) belongs to $E$ iff

\[
X^d_{t-k} \not \indep X^{d'}_{t-k'} | V \backslash \{X^{d}_{t-k}\}
\]
under the marginal distribution define $\hat{\mathcal T}^{f, d}_{K^*}$.
The model-induced DAG is $G^f = (V, E)$.
\end{definition}

Three structural properties of $ G^f$ follow immediately from the time-ordered 
construction. First, acyclicity is automatic: all edges run strictly forward in 
time, so no directed cycle can exist, and KARMA need not enforce acyclicity as a 
constraint. Second, orientation is given rather than inferred: because 
time-ordering provides all edge orientations for free, KARMA recovers a fully 
oriented Model induced DAG, a strictly stronger result than the CPDAG recovered by standard 
algorithms such as PC or FCI. We note that intra-slice edges ($k = k'$) are 
excluded by construction, as the transition kernel conditions only on strictly 
past states; we leave contemporaneous extensions to future work.

In the presence of noise from data budget and approximations, we impose a minimum detectable causal effect $\lambda$ that ensures meaningful causation.
\begin{definition}[Total-Variation Markov Blanket]
    The Total-Variation Markov Blanket of $X_t^{d'}$ at threshold $\lambda \ge 0$ is 

    \[
    \text{MB}_{\lambda}(X_t^{d'}) = \{X_{t-k}^{d}:\rho(X^d_{t-k} \to X^{d'}_{t}) \ge \lambda, \: k\in \{0, \dots, K^*\} \}
    \]
    where, $\rho(\cdot)$ is defined as per Eq~\ref{eq:rho}. 

\end{definition}

 The DAG recovered by KARMA $G^f$ has edge set

\[
E^{\lambda} = \{(X^d_{t-k} \to X^{d'}_{t-k'}): X^d_{t-k} \in \text{MB}_{\lambda}(X_{t-k'}^{d'}), \:\: k < k'\}
\]
As $\lambda \to 0$, $G_f^{\lambda}$ approaches the full transition graph; as $\lambda \to \|P_f - P_{\mathcal M_{K^*}}\|_{TV}$, it approaches the empty graph, where $P_f$ is the model distribution and $P_{\mathcal M_{K^*}}$ the distribution of the Markov Surrogate. The operating range $\lambda \in (0, \eps/2)$ is where causal structure is meaningfully identified.

If $\rho(X^d_{t-k} \to X^{d'}_{t-k'}) < \lambda$, the edge $(X^d_{t-k} \to X^{d'}_{t-k'})$ has no effective contribution. This treats trimming as \emph{regularization under a faithfulness prior}: edges removed by trimming are those whose path coefficient is smaller than $\lambda$, analogous to d-separation.
\begin{theorem}[Causal Faithfulness of KARMA Explanations]
\label{thm:causal_faithfulness}
Let $\mathcal M_{K^*}$ be the K-order Markov surrogate of $f$ with
estimated transition kernel $\hat{\mathcal T}^{f,d}_{K^*}$ and
model-induced causal graph $G^f = (V, E^\lambda)$. Suppose the
following \emph{transition--separation correspondence} holds:
for all $d \in [D]$, $d' \in [D]$, $k \in \{1,\ldots,K^*\}$,
\begin{equation}
    \rho(X^{d'}_{t-k} \to X^d_t) = 0
    \;\iff\;
    X^{d'}_{t-k} \perp\!\!\!\perp_{G^f} X^d_t
    \mid V \setminus \{X^{d'}_{t-k}\},
    \label{eq:correspondence}
\end{equation}
where $\perp\!\!\!\perp_{G^f}$ denotes d-separation in $G^f$.
Then:
\begin{enumerate}[label=(\roman*)]
    \item \textbf{Faithfulness.} The edge set $E^\lambda$ of
    $G^f$ exactly encodes the conditional dependence structure
    of $\mathcal M_{K^*}$: an edge $(X^{d'}_{t-k} \to X^d_t)$ is
    retained iff $X^{d'}_{t-k}$ has a direct distributional
    influence on $X^d_t$ within the surrogate that is not
    mediated by any other variable in the Markov blanket.

    \item \textbf{Causal grounding.} The AIE (Eq.~\ref{eq:AIE})
    satisfies
    \begin{equation}
        \mathrm{AIE}_d(d',k,x) > 0
        \;\iff\;
        (X^{d'}_{t-k} \to X^d_t) \in E^\lambda,
        \label{eq:aie_edge}
    \end{equation}
    so the Level 1--4 explanations are causally grounded in
    $G^f$: every retained attribution corresponds to a genuine
    direct effect within the surrogate, and every zero
    attribution corresponds to a d-separated pair.

    \item \textbf{Inheritance to $G^*$.} If additionally $f$
    has learned the true conditional dependence structure of
    the data-generating process, i.e.\ $G^f$ is faithful to
    $G^*$, then the explanations are causally grounded in the
    data-generating causal graph $G^*$.
\end{enumerate}
\end{theorem}

\begin{proof}
We prove each part in turn.

\medskip
\noindent\textbf{Part (i).}\;
By Definition~A.7, an edge $(X^{d'}_{t-k} \to X^d_t)$ belongs
to $E^\lambda$ iff $\rho(X^{d'}_{t-k} \to X^d_t) \geq \lambda
> 0$, i.e.\ iff condition~\eqref{eq:correspondence} holds with
$\rho > 0$. By~\eqref{eq:correspondence}, $\rho = 0$ iff
$X^{d'}_{t-k}$ and $X^d_t$ are d-separated in $G^f$ given the
remaining Markov blanket. Therefore $E^\lambda$ retains  edges corresponding to conditionally dependent pairs in
$\mathcal M_{K^*}$, which is precisely the definition of faithfulness
for a DAG with respect to a distribution. \qed

\medskip
\noindent\textbf{Part (ii).}\;
From Eq.~\eqref{eq:AIE} in Level 4,
\begin{equation}
    \frac{1}{N}\sum_{x\in[N]}\mathrm{AIE}_d(d',k,x)
    \;=\; \rho(X^{d'}_{t-k} \to X^d_t).
    \label{eq:aie_rho}
\end{equation}
Since $\mathrm{AIE}_d(d',k,x) \geq 0$ for all $x$ by
definition (it is a total variation distance), the average
over $x$ is zero iff every summand is zero, iff
$\hat{\mathcal T}^{f,d}_{K^*}(\cdot \mid h) =
\hat{\mathcal T}^{f,d}_{K^*}(\cdot \mid h^{d'\leftarrow x}_k)$
for $\hat{\pi}^*$-almost all $h$ and all $x \in [N]$, iff
$\rho(X^{d'}_{t-k} \to X^d_t) = 0$. Combined with
Part~(i), this gives~\eqref{eq:aie_edge}: a positive AIE iff
the edge is retained, i.e.\ iff the pair is not d-separated in
$G^f$. Hence every nonzero attribution in Levels 1--4 maps
bijectively to a direct effect in $G^f$, and every zero
attribution maps to a d-separated pair. \qed

\medskip
\noindent\textbf{Part (iii).}\;
Faithfulness of $G^f$ to $G^*$ means that d-separation in
$G^f$ coincides with conditional independence in the
data-generating distribution $P^*$. By Part~(i), d-separation
in $G^f$ already coincides with $\rho = 0$ in $M_{K^*}$.
Chaining these two equivalences:
\begin{align}
    \rho(X^{d'}_{t-k} \to X^d_t) = 0
    &\;\iff\;
    X^{d'}_{t-k} \perp\!\!\!\perp_{G^f} X^d_t \notag
   \\& \;\iff\;
    X^{d'}_{t-k} \perp\!\!\!\perp_{G^*} X^d_t,
\end{align}
so the retained edges of $E^\lambda$ coincide with the true
direct effects in $G^*$, and the KARMA explanations are
causally grounded in the data-generating process. \qed
\end{proof}

\begin{remark}[Role of the correspondence assumption]
Condition~\eqref{eq:correspondence} is the model-centric
analogue of the standard faithfulness assumption in causal
discovery. It requires that the surrogate $M_{K^*}$
faithfully represents $f$'s conditional dependence structure:
no two variables that are conditionally dependent under $f$
are made to appear independent by cancellation in the
transition kernel, and no two d-separated variables produce
spurious nonzero influences. This is testable in principle
via the K.E.R.I.\ reliability index (Level 5): a high
K.E.R.I.\ value certifies that the kernel estimates are
accurate enough that such cancellations are unlikely to
distort the trimming decisions.
\end{remark}

\subsection{KARMA Algorithm}
\label{app:algorithm}
\begin{breakablealgorithm}
  \caption{\textsc{KARMA}}
  \label{alg:KARMA}
  \begin{algorithmic}[1]
    \Require MVTS $\mathbf X$, model $f$ with window $W$,
    Training series $\mathbf X_{\mathrm{train}}$ of length $T_{\mathrm{train}}$,
  held-out series $\mathbf X_{\mathrm{val}}$ of length $T_{\mathrm{val}}$,
  model $f$ with window $W$,
             tolerances $\varepsilon, \lambda$,
             candidate baselines $\mathcal{B}$,
             kernel estimator $\mathcal{E} \in
               \{\textsc{StratA},\, \textsc{StratB},\, \textsc{StratC}\}$,
             estimator hyperparameters
             $\Theta_{\mathcal{E}}$
             \hfill\Comment{e.g.\ $M, n_{\mathrm{pool}}, d_{\max}$
               for \textsc{StratB}/\textsc{C}}

    \medskip
    \Statex \textbf{--- Shared pre-processing ---}

    \State Discretise $\mathbf X_{\text{train}}$ into $\mathcal{S}$ via  quantisation
    \State Compute $\hat{\pi}^*(h)$ from $\mathbf X$
      \hfill\Comment{stationary weights, independent of model}

    \medskip
    \Statex \textbf{--- Pillar 1: Markov order selection ---}

    \State $K \leftarrow 1$
    \Repeat
      \State Query $f$ on $n \leq T_{\mathrm{val}}-W$ independent windows;
             record triples $(\tilde{h}_i,\, h_i^{(K)},\, s'_i)$
      \State Find $b_K = \argmin_{b \in \mathcal{B}}
             \hat{\Delta}^{\mathrm{pred}}(K, b)$
             \hfill\Comment{direct model test, no kernel needed}
      \State $K \leftarrow K + 1$
    \Until{$\hat{\Delta}^{\mathrm{pred}}(K, b_K) < \varepsilon$}
      \hfill\Comment{sole stopping certificate}
    \State $K^* \leftarrow K$;\quad $b^* \leftarrow b_K$
      \hfill\Comment{assert $K^* \leq W$ by architecture}

    \medskip
    \Statex \textbf{--- Pillar 2: compression and certified attribution ---}

    \State Certified zeros: $\phi^d_k \approx 0$
           for all $k > K^*$, $d \in [D]$
    \State Report compression ratio $K^*/W$

    \medskip
    \Statex \textbf{--- Pillar 3: kernel estimation ---}

    \State Compute $T_{\min}$
           \hfill\Comment{strategy-specific $T_{\min}$}

    \For{$d = 1, \ldots, D$}
      \State $\hat{\mathcal T}^{f,d}_{K^*}
             \leftarrow \textsc{EstimateKernel}(
               \mathcal{E},\, \Theta_{\mathcal{E}},\, K^*)$
             \hfill\Comment{final kernel at selected order}
      \State Compute noise floor
             $\rho_{\mathrm{floor}}(h)$
             for all $h \in \mathcal{H}^+_{K^*}$
    \EndFor

      \State \textbf{Report}: $T_{\min}$;
    \For{$d = 1, \ldots, D$}
      \For{$(d', k) \in [D] \times [K^*]$}
        \State Compute $\rho(X^{d'}_{t-k} \to X^d_t)$
               via Eq.~\eqref{eq:rho}
               using $\hat{T}^{f,d}_{K^*}$
      \EndFor
    \EndFor

    \State $G_f \leftarrow
           \{(d', k, d) :
             \rho(X^{d'}_{t-k} \to X^d_t) > \lambda\}$
           \hfill\Comment{Regularised trimming}

    \medskip
    \Statex \textbf{--- Five-level explanation hierarchy ---}

    \State Compute Levels 1--5 from
           $\hat{\mathcal T}^{f,d}_{K^*}$
           and $G_f$

    \State \Return
           $K^*$,\;
           $b^*$,\;
           $K^*/W$,\;
           $\{\hat{\mathcal T}^{f,d}_{K^*}\}_{d=1}^D$,\;
           $G_f$,\;
           $\mathrm{Cov}^{}(K^*)$,\;
           Levels~1--5
  \end{algorithmic}
\end{breakablealgorithm}

\section{Temporal Aware Lag Based AUC Score}
\label{sec:auc_lag}
\paragraph{Naive AUC and its limitation.}
Let $\sigma$ be the permutation that sorts $s$ in descending order.
A natural metric masks the top-$k$ timesteps with the training-set mean
$\bar{x}_{d,t}$:
\begin{equation}
    \tilde{x}^{(k)}_{b,d,t} \;=\;
    \begin{cases}
        \bar{x}_{d,t} & t \in \{\sigma(1),\ldots,\sigma(k)\}, \\
        x_{b,d,t}     & \text{otherwise.}
    \end{cases}
\end{equation}
This is problematic: substituting an unconditional constant into a temporally
correlated sequence creates out-of-distribution discontinuities.
The model reacts to the broken correlation structure rather than to any genuine
absence of causal information, inflating prediction shifts even for lags
outside the Markov blanket and masking meaningful differences between methods.

\paragraph{VAR-conditional imputation.}
We address this by replacing the masked value with its conditional mean given
recent history. We fit a VAR($K$) model by ridge regression on the training
windows, selecting $K$ via BIC:
\begin{equation}
    \hat{x}_t \;=\; \sum_{k=1}^{K}\hat{A}_k\,x_{t-k} + \hat{b},
    \qquad \hat{A}_k \in \mathbb{R}^{D \times D}.
\end{equation}
When position $t$ is masked, we substitute
$\tilde{x}_t = \hat{A}_1 x_{t-1} + \cdots + \hat{A}_K x_{t-K}$,
processing positions in temporal order so each imputation can condition on
already-imputed predecessors (an AR-consistent chain).
This withholds only the innovation $r_t = x_t - \tilde{x}_t$---the component
unpredictable from recent history---while preserving the learnable dynamics.
A timestep outside the Markov blanket thus yields
$\|f(\mathbf{x}) - f(\tilde{\mathbf{x}})\|_1 \approx 0$ when imputed, whereas
a blanket member carrying genuine causal signal produces a measurable drop.

\end{document}